\documentclass[twoside]{article}

\usepackage[accepted]{aistats2021}
\usepackage[round]{natbib}

\usepackage[utf8]{inputenc} 
\usepackage[T1]{fontenc}    
\usepackage{hyperref}       
\usepackage{url}            
\usepackage{booktabs}       
\usepackage{amsfonts}       
\usepackage{nicefrac}       
\usepackage{microtype}      

\usepackage{multicol}
\usepackage{graphicx}
\usepackage{subcaption}
\usepackage[labelfont = bf, font = small]{caption}
\usepackage{tikz}
\usepackage{amssymb}
\usepackage{amsmath}
\usepackage{amsfonts}     
\usetikzlibrary{bayesnet}
\usepackage{amsthm}
\usepackage{nicefrac}
\usepackage[linesnumbered,ruled,vlined]{algorithm2e}
\usepackage{setspace}
\usepackage{bbm}
\SetKwInput{KwInput}{Input}                
\SetKwInput{KwOutput}{Output}              

\newcommand{\R}{\mathbb{R}}
\newcommand{\E}{\mathbb{E}}

\newcommand{\bmu}{\pmb{\mu}}
\newcommand{\trace}{\mathbf{tr}}
\newcommand{\calS}{\mathcal{S}}
\newcommand{\calF}{\mathcal{F}}
\newcommand{\calP}{\mathcal{P}}
\newcommand{\calA}{\mathcal{A}}
\newcommand{\calN}{\mathcal{N}}
\newcommand{\calD}{\mathcal{D}}
\newcommand{\calQ}{\mathcal{Q}}
\newcommand{\calR}{\mathcal{R}}
\newcommand{\Zu}{{Z_{\mathbb{I}_U}}}
\newcommand{\Zs}{{Z_{\mathbb{I}_S}}}
\newcommand{\Zx}{{Z_{\mathbb{I}_x}}}
\newcommand{\Zy}{{Z_{\mathbb{I}_y}}}
\newcommand{\Xu}{{X_{\mathbb{I}_U}}}
\newcommand{\Xs}{{X_{\mathbb{I}_S}}}
\newcommand{\Xx}{{X_{\mathbb{I}_x}}}
\newcommand{\Xy}{{X_{\mathbb{I}_y}}}
\newcommand{\Gu}{{G_{\mathbb{I}_U}}}
\newcommand{\Gs}{{G_{\mathbb{I}_S}}}
\newcommand{\Gx}{{G_{\mathbb{I}_x}}}
\newcommand{\Gy}{{G_{\mathbb{I}_y}}}
\newcommand{\Spairs}{\calS_{\text{pairs}}}
\newcommand{\Sigmag}{\Sigma^{(g)}}
\newcommand{\Sigmaeff}{\Sigma_{\text{eff}}}
\newcommand{\maxeig}{{\max \text{eig}}}
\newcommand{\leff}{l_{\text{eff}}}
\newcommand{\Is}{{\mathbb{I}_S}}

\newcommand{\Ix}{\mathbb{I}_x}
\newcommand{\Iy}{\mathbb{I}_y}

\DeclareMathOperator*{\argmax}{arg\,max}
\DeclareMathOperator*{\argmin}{arg\,min}

\newtheorem{theorem}{Theorem}[section]
\newtheorem{corollary}{Corollary}[theorem]
\newtheorem{lemma}[theorem]{Lemma}

\theoremstyle{definition}
\newtheorem{definition}{Definition}[section]

\usepackage{hyperref}

\begin{document}

%

%

\twocolumn[

\aistatstitle{Location Trace Privacy Under Conditional Priors}

\aistatsauthor{ Casey Meehan \\ \texttt{cmeehan@eng.ucsd.edu} \And Kamalika Chaudhuri  \\ \texttt{kamalika@eng.ucsd.edu}}

\aistatsaddress{ UC San Diego \And  UC San Diego } ]

\begin{abstract}
  Providing meaningful privacy to users of location based services is particularly challenging when multiple locations are revealed in a short period of time. This is primarily due to the tremendous degree of dependence that can be anticipated between points. We propose a R\'enyi divergence based privacy framework for bounding expected privacy loss for conditionally dependent data. Additionally, we demonstrate an algorithm for achieving this privacy under Gaussian process conditional priors. This framework both exemplifies why conditionally dependent data is so challenging to protect and offers a strategy for preserving privacy to within a fixed radius for sensitive locations in a user's trace. 
\end{abstract}

\section{Introduction}
\label{sec:introduction}

Location data is acutely sensitive information, detailing where we live, work, eat, shop, worship, and often when, too. Yet increasingly, location data is being uploaded for smartphone services such as ride hailing and weather forecasting and then being brokered in a thriving user location aftermarket to advertisers and even investors \citep{nyt}. Users share location `traces' when they release a sequence of locations, often across a short period of time. These traces are then used by central servers to monitor traffic trends, track individual fitness, target marketing, and even to study the effectiveness of social-distancing ordinances \citep{wash_post}. Here, we aim to provide a \emph{local} privacy guarantee, wherein traces are sanitized at the user level before being transmitted to a centralized service. Note that this requires different guarantees and mechanisms than in \emph{aggregate} applications making queries on large location trace databases. 

Specifically, we guarantee a radius $r$ of privacy at any sensitive time point or combination of time points within a given trace. This is challenging due to the fact that the locations within traces are highly inter-dependent. Informally, traces tend to follow relatively smooth trajectories in time. If not sanitized carefully, that knowledge alone may be exploited to infer actual locations from the released version of the trace. This work centers on designing meaningful privacy definitions and corresponding mechanisms that takes this dependence into account.

Broadly speaking, the vast majority of prior work on rigorous data privacy can be divided into two classes that differ by the kind of guarantee offered: differential and inferential privacy. Differential privacy (DP) guarantees that the participation of a single person in a dataset does not change the probability of any outcome by much. In contrast, inferential privacy guarantees that an adversary who has a certain degree of prior knowledge cannot make certain sensitive inferences.

DP for releasing aggregate statistics of a spatio-temporal dataset has been well studied \citep{traffic_monitoring, quantifying_dp_cao, bayesian_DP, dependent_dp}. There, the idea is to add enough noise to released statistics such that the effect of any user's participation is obscured, even if their locations are highly correlated to each other or to those of other users. Here, such a guarantee does not apply since we aim to release a sanitized version of a single user's trace.

In this local case we cannot rule out the possibility that the data curator knows who each individual is and who participated. Instead, we want to guarantee that event level information \emph{about} each trace remains private. In this work, at any sensitive time $t$ we mask whether the user visited location A or location B for any A,B less than $r$ apart. Without \emph{ad hoc} modifications, standard DP tools are insufficient for achieving this for the primary reasons that 1) the domain of location is virtually unbounded and 2) locations are highly dependent across a short period of time. To see this, consider the following instinctual approaches to achieving location trace privacy. 

\paragraph{Approach A:} apply Local Differential Privacy (LDP) to each trace. Imagine a dataset of traces, each from a separate individual. Applying LDP implies that every trace has nearly the same probability of releasing the same sanitized version. This would be robust to arbitrary side information about dependence between locations in any one trace. Unfortunately, the amount of additive noise needed to achieve this would destroy nearly all utility: sanitized traces from California would have almost the same probability of showing up in Connecticut as do those from New York. Even if we constrained the domain to just Manhattan, this definition would not permit enough utility to perform e.g. traffic monitoring. 

\paragraph{Approach B:} apply LDP to each location within a trace. To preserve some utility, imagine a single trace as a dataset of $n$ locations, each of which enjoys $\varepsilon$-LDP guarantees. This alone is not robust to arbitrary dependence between locations. By the logic of group LDP, it does satisfy $k \varepsilon$-LDP regardless of the dependence between any $k$ locations. This approach has two setbacks. First, how to set $k$ is unclear. Technically, all points in the trace are correlated, so to ward off worst-case correlations one might set it to the length of the trace, which is identical to Approach A. Second, even if location is bounded to a single city or county, satisfying this definition would still destroy nearly all utility. We cannot use sanitized traces for traffic monitoring if locations from either side of town have about same probability of being sanitized to the same value. 

\paragraph{Approach C:} apply LDP guarantees to each location within a trace, but only within any region less than width $r$. This definition is known as Geo-Indistinguishability (GI) \citep{GI}. GI provides a substitute for restricting the domain of location allowing us to salvage some utility. Here, only locations within $r$ of each other are required to have $\varepsilon$-LDP guarantees. In DP parlance, we might say that `neighboring traces' have one location altered by $\leq r$ and are identical everywhere else. This gives us the guarantee we want for a trace with one location, but not with more than one location. To see why, compare with Approach B. Analogously, $(\varepsilon, r)$-GI along a trace provides $(k \varepsilon, r)$-GI to any subset of $k$ locations. Like Approach B, setting $k$ is unclear. Yet unlike Approach B, GI is not resistant to arbitrary dependence between any $k$ locations. Any dependence where a change in one or more location(s) by $r$ implies a change in some other location(s) by $\geq r$ breaks the GI guarantee. Even with the simplest models of dependence (e.g. if we know the true trace ought to move in a straight line) this is a problem. 

To reiterate, applying LDP to traces or to locations within traces (Approaches A \& B) does not provide a principled method for meaningful privacy with reasonable utility. GI adapts LDP by giving guarantees only within a radius $r$. But in relaxing LDP, GI compromises the standard DP tools for handling obvious dependences between data-points like group DP. In our eyes, this warrants an \emph{inferentially private} approach. Here, we continue to provide privacy within a radius $r$, thus allowing for utility. Yet instead of providing resistance to arbitrary dependence across any $k$ locations, we aim to provide resistance to natural models of dependence between all locations. One may view such models as an adversary's prior beliefs about what traces are likely, like the straight-line prior mentioned earlier.

In contrast with differential privacy, providing inferential privacy guarantees is more complex, and has been less studied. It is however appropriate for applications such as ours, where information must be released based on a single person's data, the features of which are private and dependent. \cite{pufferfish} provide a formal inferential privacy framework called Pufferfish, and design mechanisms for specific Pufferfish instances. As these instances do not apply to our setting, we adapt the Pufferfish framework to location privacy and more broadly to releasing any sequence of real-valued private information. 

\begin{figure*}[h]
	\centering
	\begin{subfigure}[b]{.45\textwidth}
		\centering
     \scalebox{0.8}{\begin{tikzpicture}

  \node[obs]                       (x1) {$X_1$};
  \node[latent, below=of x1]			  (z1) {$Z_1$};
  \node[obs, right=of x1]             (x2) {$X_2$};
  \node[latent, below=of x2]			  (z2) {$Z_2$};
  \node[latent, right=of x2]          (x3) {$X_3$};
  \node[latent, below=of x3]			  (z3) {$Z_3$};
  \node[latent, right=of x3]          (x4) {$X_4$};
  \node[latent, below=of x4]			  (z4) {$Z_4$};

  \edge {x1} {z1};
  \edge {x2} {z2}; 
  \edge {x3} {z3};
  \edge {x4} {z4};
  \edge [-] {x1} {x2}; 
  \edge [-] {x2} {x3};
  \edge [-] {x3} {x4};
  \draw (x1) edge[-, bend left = 40] node [right] {} (x3);
  \draw (x1) edge[-, bend left = 40] node [right] {} (x4);
  \draw (x2) edge[-, bend left = 40] node [right] {} (x4);


\end{tikzpicture}}  

		\caption{}
		\label{fig:full model}
	\end{subfigure}
	\begin{subfigure}[b]{.45\textwidth}
		\centering
     \scalebox{0.8}{\begin{tikzpicture}

  \node[obs]                       (s) {$\Xs$};
  \node[latent, below=of x1]			  (zs) {$\Zs$};
  \node[latent, right=of x1]             (u) {$\Xu$};
  \node[latent, below=of x2]			  (zu) {$\Zu$};

  \edge {s} {zs};
  \edge {u} {zu}; 
  \edge [-] {s} {u}; 


\end{tikzpicture}}  

		\caption{}
		\label{fig:condensed model}
	\end{subfigure}
	\caption{(a) An example graphical model of a four point trace $X$. (b) The more general grouped version of the model in (a), with the secret set $\Xs = \{X_1, X_2\}$ and the remaining set $\Xu = \{X_3, X_4\}$. 
	}
	\label{fig:graphical models}
\end{figure*}
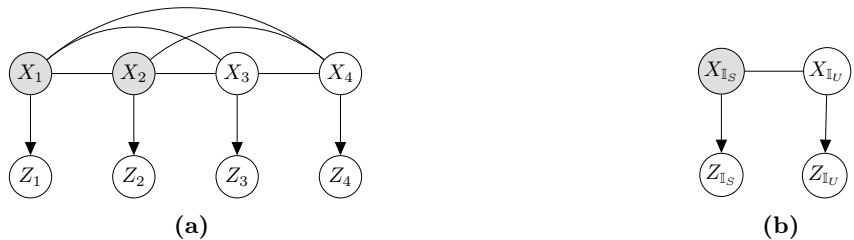

\paragraph{Contributions:}In this work, we propose an inferentially private approach to guaranteeing a radius $r$ of privacy for sensitive points in location traces in three parts: 
\begin{itemize}
	\item First, we propose an adaptable privacy framework tailored to sequences of highly dependent datapoints that adapts Pufferfish privacy \citep{pufferfish} to use R\'enyi Differential Privacy (RDP) \citep{renyi}. Given a model of dependence between points, this framework more appropriately estimates the risk of inference within radius $r$ on points of interest than do vanilla LDP approaches. 
	\item We then demonstrate how to implement our framework for the highly flexible and expressive setting of Gaussian process (GP) priors. These nonparametric models capture the spatiotemporal aspect of location data \citep{PCS_GP, ATM_GP, Traffic_GP}. GPs have a natural synergy with R\'enyi privacy enabling an interpretable upper bound on privacy loss for additive Gaussian privacy mechanisms (that add Gaussian noise to each point). Using this, we design a semidefinite program (SDP) that optimizes the correlation of such mechanisms to minimize privacy loss without destroying utility, efficiently thwarting the inference of sensitive locations. 
	\item Finally, we provide experiments on both location trace and home temperature data to demonstrate the advantage of these techniques over Approach C mechanisms like GI. We find that our mechanisms successfully obscure sensitive locations while respecting utility constraints, even when the prior model is misspecified. 
\end{itemize}
 
Ultimately, by resisting only reasonable kinds of dependence in the data we are able to offer both meaningful privacy and utility. We show that our framework is robust to misspecification of this reasonable dependence and offers a privacy loss that is both tractable and interpretable.

\section{Preliminaries and Problem Setting}
\label{sec:preliminaries}
A user transmits a sequence of $N$ 2-dimensional locations along with their corresponding timestamps, collectively forming a `trace'. We `unroll' the trace into $n$ real-valued random variables $X = \{X_1, X_2, \dots, X_n\}$. A trace of 10 2d locations has $n = 2 \times 10 = 20$ random variables $X_i$. Instead of releasing the raw trace $X$, the user releases a private version $Z = \{Z_1, Z_2, \dots, Z_n\}$, by way of an additive noise mechanism $Z = X + G$, where $G = \{G_1,G_2, \dots, G_n\}$ is random noise produced by a privacy mechanism.

An adversary, receiving the obscured trace $Z$, then reasons about the true locations at some sensitive time(s). To reference the sensitive times, we use index set $\Is$. If the sensitive indices are $\Is = \{1,2\}$, the corresponding location values are $\Xs = \{X_{1}, X_{2}\}$ (e.g. referring to the two coordinates of one location). When inferring the true value of $\Xs$, the adversary makes use of the remaining points in the trace at indices $\mathbb{I}_U = [n] \backslash \Is$, denoted $\Xu$, with obscured values $\Zu$. This separation of points into $\Xs$ and $\Xu$ is represented in \textbf{Figure \ref{fig:graphical models}}. 

We use location as a guiding example, but such inter-dependent traces $X$ could take the form of home temperature time series data or spatial data like 3D facial maps used for identification. Going forward, we will continue to denote $X = \{X_1, X_2, \dots, X_n\}$ with the understanding that \emph{any} subsequence of $d$ points e.g. $\Xs = \{X_2, X_6, \dots\}$ could represent a $d$-dimensional sensitive value, or $Nd$ points could represent $N$ $d$-dimensional sensitive values. 

For the real-valued distributions considered here, $P_{\times}( \bullet )$ refers to a density of distribution $\times$ on r.v. $\bullet$ and $P_{\times}(\bullet | *)$ is its regular conditional density given $*$. 

\subsection{Background}
GI limits what can be inferred about the sensitive $\Xs$ from its corresponding $\Zs$, but not from the remaining locations $\Zu$. To do so we need a privacy definition that specifies what events of random variable $\Xs$ we wish to obscure, which realistic priors of inter-dependence to protect against, and a privacy loss. 

\subsection{Basic and Compound Secrets}

We borrow heavily from the Pufferfish framework \citep{pufferfish}, and specialize it for the setting of location traces. We define our own set of \emph{secrets} --- the collection of events we wish to obscure --- and \emph{discriminative pairs}, the pairs of secret events we do not want an adversary to tell between. 

\paragraph{Basic Secrets \& Pairs} 
After releasing $Z$, we do not want an adversary with a reasonable prior on $X$, $\calP \in \Theta$, to have sharp posterior beliefs about the user's location at some sensitive time (e.g. one of the sensitive times in \textbf{Figure \ref{fig:nyc_example}} of Appendix \ref{apx: Illustrations}). As such, the adversary cannot distinguish whether the user visited location A or some nearby location B at that time. Let $x_s\in \R^2$ represent a possible assignments to $\Xs$, hypothesizing the true sensitive location. Any such assignment is secret, $\calS = \{ \Xs = x_s  : x_s \in \R^2\}$. Specifically, we want the posterior probability of any two assignments to $\Xs$ within a radius $r$ to be close: $\Spairs = \{(x_s, x_s'):\|x_s - x_s'\|_2 \leq r\}$. This protects a single time within a trace of locations. More generally, in the context of spatiotemporal data of any dimension, we call this a \emph{basic secret}. 

\paragraph{Compound Secrets \& Pairs} 
Suppose we have three sensitive times (again as in \textbf{Figure \ref{fig:nyc_example}}). A mechanism that blocks inference on each of these separately does not prevent inference on the combination of them simultaneously. To obscure hypotheses on \emph{all three} of these, we modify our set of secrets to any combination of assignments to each secret location: 
\begin{align*}
	\calS = \big\{ \{\Xs_1 = x_{s1}\} \cap \{\Xs_2 = x_{s2}\} \cap \{\Xs_3 = x_{s3} &\} \\
	: x_{si} \in \R^2, i \in [3] &\big\} \ .
\end{align*} 
Now, the set of discriminative pairs is any two assignments to all three secret locations: 
\begin{align*}
	\Spairs = \Big\{\big( \{x_{s1}, x_{s2}, x_{s3}\} &, \{x_{s1}', x_{s2}', x_{s3}'\}\big) \\
	&: \| x_{si} - x_{si}' \|_2 \leq r, \ i \in [3] \Big\}
\end{align*}
This protects against compound hypotheses: if daycare and work are within $r$ of each other, this keeps an adversary from inferring $\Xs_1 = $ `daycare' \emph{and} $\Xs_2 = $ `work' versus $\Xs_1 = $ `work' \emph{and} $\Xs_2 = $ `daycare'. More generally, in the context of spatiotemporal data of any dimension, we call this a \emph{compound secret}. Intuitively, a mechanism that protects a compound secret of locations close together in time prevents a Bayesian adversary from leveraging the remainder of the trace to infer direction of motion at those sensitive times. Note that bounding the privacy loss of a compound secret does not bound the privacy loss of its constituent basic secrets.

Going forward, we refer to $\Is$ as the `secret set'. 

\subsubsection{Gaussian Processes}
For the purpose of location privacy, it is important to choose a prior class $\Theta$ such that the conditional distribution $P_\calP(\Xu | \Xs)$ is simple to compute for any secret set $\Is$ and any prior $\calP \in \Theta$. Of course, it is also critical that the prior class naturally models the data, and thus consists of `reasonable assumptions' for adversaries. GPs satisfy both these requirements. We model a full $d$-dimensional trace sampled at $N$ times by `unrolling' it into a $n = dN$ dimensional GP. 
\begin{definition}\emph{Gaussian process} 
	A trace $X$ is a Gaussian process if $X_{\mathbb{I}_M}$ has a multivariate normal distribution for any set of indices $\mathbb{I}_M \subset [n]$. If $X$ is a gaussian process, then the function $i \rightarrow \E[X_i]$ is called the mean function and the function $(i,j) \rightarrow \text{Cov}(X_i, X_j)$ is called the kernel function. 
\end{definition}
In this work, the kernel uses locations' time stamps to compute their covariance $(t_i, t_j) \rightarrow \text{Cov}(X_i, X_j)$, but generally could use any side information provided with each location. 


GPs have simple, closed form conditional distributions. Let $X \sim \calN(\mu, \Sigma)$, where $\mu \in \R^{n}$ and $\Sigma \in \R^{{n} \times {n}}$. Then, the random variable $\Xu | \{\Xs = x_s\} \sim \calN(\mu_{u|s}, \Sigma_{u|s})$, where $\mu_{u|s} = \mu_u + \Sigma_{us} \Sigma_{ss}^{-1} (x_s - \mu_s)$ and $\Sigma_{u|s} = \Sigma_{uu} - \Sigma_{us}\Sigma_{ss}^{-1} \Sigma_{su}$. Here, $\mu_s$ denotes the mean vector $\mu$ accessed at indices $\Is$ and $\Sigma_{su}$ denotes the covariance matrix $\Sigma$ accessed at rows $\Is$ and columns $\mathbb{I}_U$. 

For GP priors, we will use additive noise $G \sim \calN(\mathbf{0}, \Sigmag)$. Thus $Z = X + G$, too, is multivariate normal. Furthermore, the distribution of any set of variables conditioned on any other set of variables in \textbf{Figure \ref{fig:graphical models}} belongs to some multivariate normal distribution.

GPs have been shown to successfully model mobility \citep{Traffic_GP, PCS_GP, ATM_GP}, even in the domain of surveillance video \citep{surveillance_GP}.  Furthermore, although these non-parametric models are characterized by second order statistics, GPs are capable of complexity rivaling that of deep neural networks \citep{Deep_NN_GP}, allowing for scalability to more complex models and domains. Our proposed results and algorithms may be applied regardless of the complexity of the chosen GP.

\subsubsection{R\'enyi Differential Privacy}
\label{sec:renyi_dp}
In the following section, we propose a privacy definition that adapts R\'enyi Differential Privacy (RDP) \citep{renyi} to the Pufferfish framework. RDP resembles Differential Privacy \citep{DP}, except instead of bounding the maximum probability ratio or \emph{max divergence} of the distribution on outputs for two neighboring databases, it bounds the \emph{R\'enyi divergence} of order $\lambda$, defined in Equation \eqref{eqn: renyi} for distributions $\calP_1$ and $\calP_2$. The R\'enyi divergence bears a nice synergy with Gaussian processes. If $\calP_1 = \calN(\mu_1, \Sigma)$ and $\calP_2 = \calN(\mu_2, \Sigma)$ --- two mean-shifted normal distributions --- the R\'enyi divergence takes on a simple closed form shown in Equation \eqref{eqn: normal renyi}. 
\begin{align}
	\label{eqn: renyi}
	D_\lambda \binom{\calP_1}{\calP_2} 
	&= \frac{1}{\lambda - 1} \log \E_{x \sim \calP_2} \Big( \frac{P_{\calP_1}(X = x)}{P_{\calP_2}(X = x)} \Big)^\lambda \\
	\label{eqn: normal renyi}
	&= \frac{\lambda}{2} (\mu_1 - \mu_2)^\intercal \Sigma^{-1} (\mu_1 - \mu_2)
\end{align}
We will make use of this in defining and bounding privacy loss in the next section. 

\section{Conditional Inferential Privacy}
We now propose a privacy framework that is tailored to sequences of correlated data, Conditional Inferential Privacy (CIP). CIP guarantees a radius $r$ of indistinguishability for the basic or compound secrets associated with any secret set $\Is$. Specifically, CIP protects against any adversary with a specific prior on \emph{the shape} of the trace, and is agnostic to their prior on the absolute location of the trace. We call the set of such prior distributions a Conditional Prior Class.

\begin{definition} \emph{Conditional Prior Class}
\label{def: conditional prior class}
	For $X = \{X_1, \dots, X_n\}$, prior distributions $\calP_i, \calP_j$ on $X$ are said to belong to the same conditional prior class $\Theta$ if a constant shift in the conditioned $x_s$ results in a constant shift on the distribution of $\Xu$. Formally, if conditional distributions $P_{\calP_i}(\Xu | \Xs = x_s) = P_{\calP_j}(\Xu + c_{ij\Is}^u  | \Xs = x_s + c_{ij\Is}^s )$ for all $x_s$.
\end{definition}

For instance, prior $P_{\calP_i}$ may concentrate probability on traces passing through Los Angeles, while $P_{\calP_j}$ concentrates on traces passing through London. Conditioning on each secret in the pair $(x_s, x_s')$ in L.A. is analogous to conditioning on each secret in the pair $(x_s + c_{ij\Is}^s, x_s' + c_{ij\Is}^s)$ in London. The corresponding pair of conditional distributions on $\Xu$ in London ($P_{\calP_j}$) are copies of those in L.A. ($P_{\calP_i}$) shifted by $c_{ij\Is}^u$. What matters is that the set of all pairs of conditional distributions under $P_{\calP_i}$ induced by secret pairs $(x_s, x_s')$ is identical to those under $P_{\calP_j}$ up to a mean shift. See Appendix \ref{apx: GP prior class} for a more detailed discussion of conditional prior classes.  
 

\begin{definition} \emph{ $(\varepsilon, \lambda)$-Conditional Inferential Privacy $(\Spairs, r, \Theta)$}
	Given compound or basic discriminative pairs $\Spairs$ associated with $\Is$, a radius of privacy $r$, a conditional prior class, $\Theta$, and a privacy parameter, $\varepsilon > 0$, a privacy mechanism $Z = \calA(X)$ satisfies $(\varepsilon, \lambda)$-CIP$(\Spairs, r, \Theta)$ if for all $(s_i, s_j) \in \Spairs$, 
	and all prior distributions $\calP \in \Theta$, where $P_\calP(s_i), P_\calP(s_j) > 0$, 
	\begin{align}
		\label{eqn:CIP loss}
		D_\lambda \binom{P_{\calA, \calP} (Z | \Xs = s_i)}{P_{\calA, \calP} (Z | \Xs = s_j)} &\leq \varepsilon
	\end{align}
\end{definition}

CIP departs from DP type notions of privacy like Approaches A$\rightarrow$C primarily by resisting only a restricted class of inter-dependence --- the conditional prior class --- as opposed to arbitrary dependence of any $k$ locations. Unlike approaches A and B, we are able to preserve utility for tasks like traffic monitoring. Unlike approach C, CIP is still resistant to realistic models of location inter-dependence. 

While this definition borrows heavily from the Pufferfish framework, it has a few key modifications. Pufferfish is generally described from a central, not local model. We specialize the kinds of secrets and discriminative pairs for the case of local location trace privacy. Additionally, we specialize the type of prior distribution class needed for this local setting: the conditional prior class. Finally, we relax the strict max divergence (max log odds) criterion of the Pufferfish definition to a R\'enyi divergence. This guarantees that --- with high probability on draws of \emph{realistic} traces $Z|\Xs$ --- the log odds will be bounded by $\varepsilon$. As $\lambda \rightarrow \infty$, the log odds are bounded for all traces, i.e. the max divergence is bounded. We formalize this in Theorem \ref{thm: prior-posterior}. 

The R\'enyi criterion of CIP greatly improves its flexibility. Unlike the standard DP Approaches A$\rightarrow$C which only take probabilities over the mechanism, we do not have full control over the randomness at play: it is partially from $\calA$ defined by us and from $\calP$ intrinsic to the data. Unlike max divergence, R\'enyi divergence is available in closed form for many distributions, allowing for a more flexible privacy framework. The $\lambda$ parameter helps us tune how strict a CIP definition is and how much noise we need to add. This allows us to design mechanisms that are resistant to natural models of dependence while preserving utility. 


\subsection{Properties}

We now identify key properties that make the CIP guarantee interpretable and robust. 

\paragraph{Interpretability:} CIP guarantees that a Bayesian adversary with any prior distribution on traces $\calP$ in the conditional prior class $\Theta$ does not learn much about basic or compound secrets from the released trace $Z$. For basic secrets, this means that the adversary's posterior beliefs regarding sensitive location $\Xs$ are not much sharper than their prior beliefs before witnessing $Z$.  
\begin{theorem} \emph{Prior-Posterior Gap:} 
\label{thm: prior-posterior}
	An $(\varepsilon, \lambda)$-CIP mechanism with conditional prior class $\Theta$ guarantees that for any event $O$ on sanitized trace $Z$
	\begin{align*}
		\bigg| \log \frac{P_{\calP, \calA}(s_i | Z \in O)}{P_{\calP, \calA}(s_j | Z \in O)} - \log \frac{P_{\calP}(s_i)}{P_{\calP}(s_j)} \bigg| \leq \varepsilon'
	\end{align*}
	for any $\calP \in \Theta$ with probability $\geq 1 - \delta$ over draws of $Z|\Xs=s_i$ or $Z|\Xs=s_j$, where $\varepsilon'$ and $\delta$ are related by
	\begin{align*}
		\varepsilon' = \varepsilon + \frac{\log \nicefrac{1}{\delta}}{\lambda - 1} \ .
	\end{align*}
	This holds under the condition that $Z|\Xs = s_i$ and $Z|\Xs = s_j$ have identical support. 
\end{theorem}
A CIP mechanism depends only on the conditional prior describing the data, not the data itself. Suppose an adversary's prior beliefs on $\Xs$ are uniform over some region. For $\lambda = 5$ and $\varepsilon = 0.1$, there is only a $\approx 1\%$ chance that their posterior odds on $s_i,s_j$ will be more than 3.5, and a $\approx 10\%$ chance that they will be more than 2. This `chance' is over draws of likely remaining locations $\Xu$ and the additive noise $G$.  Proofs of all results are in Appendix \ref{apx: proofs}.

For additive noise mechanisms like $\calA(X) = X + G = Z$, the CIP loss can be split into two terms: one accounting for the direct privacy loss of $\Zs$ on $\Xs$ and a second accounting for the inferential privacy loss of $\Zu$ on $\Xs$ via $\Xu$.

\begin{lemma}\emph{Conditional Independence}
	\label{lem: renyi additive loss}
	For an additive noise mechanism, a fully dependent trace as in \textbf{Figure \ref{fig:full model}}, and any prior $\calP$ on $X$ the CIP loss may be expressed as
	\begin{align}
	\label{eqn:two terms}
		&D_\lambda \binom{P_{\calA, \calP}(Z | \Xs = s_i)}{P_{\calA, \calP}(Z | \Xs = s_j)}  \\ 
		\vspace{2em}
		&= \sum_{i \in \Is} \bigg[ D_\lambda \binom{P_\calA(Z_i | X_i = s_i)}{P_\calA(Z_i | X_i = s_j)} \bigg]
		+ D_\lambda \binom{P_{\calA, \calP}(\Zu | \Xs = s_i)}{P_{\calA, \calP}(\Zu | \Xs = s_j)} \notag
	\end{align}
\end{lemma} 
One interpretation of GI is that it assumes all locations $X_i$ are independent. In this case, the second term vanishes and the privacy loss only depends on randomness of the mechanism, not the prior. 

\paragraph{Robustness:}
\cite{no_free_lunch} show that it is impossible to achieve both utility and privacy resistant to all priors. CIP provides resistance to a reasonable class of priors $\calP \in \Theta$, but it is possible that the true distribution $\calQ \notin \Theta$. In this case, the privacy guarantees degrade gracefully as the divergence between $\calQ$ and $\calP \in \Theta$ grows. 
\begin{theorem}\emph{Robustness to Prior Misspecification}
\label{thm: prior misspecification}
	Mechanism $\calA$ satisfies $\varepsilon(\lambda)$-CIP for prior class $\Theta$. Suppose the finite mean true distribution $\calQ$ is not in $\Theta$. The CIP loss of $\calA$ against prior $\calQ$ is bounded by 
	\begin{align*}
		D_\lambda \binom{P_{\calA, \calQ}(Z | \Xs = s_i)}{P_{\calA, \calQ}(Z | \Xs = s_j)} \leq \varepsilon'(\lambda)
	\end{align*}
	where
	\begin{align*}
		\varepsilon'(\lambda) 
		&= \frac{\lambda - \frac{1}{2}}{\lambda - 1} \ \Delta(2\lambda) + 
		\Delta(4\lambda - 3) +
		\frac{2\lambda - \frac{3}{2}}{2\lambda - 2} \ \varepsilon(4 \lambda -2)
	\end{align*}
	and where $\Delta(\lambda)$ is
	\begin{align*}
		\inf_{\calP \in \Theta} \sup_{s_i \in \calS} \max \bigg\{ 
		D_\lambda \binom{P_{ \calP}(\Xu | \Xs = s_i)}{P_{ \calQ}(\Xu | \Xs = s_i)}, 
		D_\lambda \binom{P_{ \calQ}(\Xu | \Xs = s_i)}{P_{ \calP}(\Xu | \Xs = s_i)}
		\bigg\}
	\end{align*}
\end{theorem}

As long as the conditional distribution on $\Xu|\Xs = s_i$ of prior $\calQ$ is close to that of some $\calP \in \Theta$, the privacy guarantees should change only marginally. This bound is tightest when $\varepsilon(\lambda)$ does not grow quickly with order $\lambda$.

%
%
%
%
%

\subsection{CIP for Gaussian Process Priors}
\label{sec: CIP for GP} 
A \emph{GP conditional prior class} is the set of all GP prior distributions with the same kernel function $(i,j) \rightarrow \text{Cov}(X_i, X_j)$ and any mean function $i \rightarrow \E[X_i]$. With an additive Gaussian mechanism $G \sim \calN(\mathbf{0}, \Sigmag)$, the CIP loss of Equation \eqref{eqn:two terms} can be bounded for any GP conditional prior class. See Appendix \ref{apx: GP prior class} for further discussion of the GP conditional prior class. 

\begin{theorem}\emph{CIP loss bound for GP conditional priors:}
\label{thm:GP bound}
	Let $\Theta$ be a GP conditional prior class. Let $\Sigma$ be the covariance matrix for $X$ produced by its kernel function. Let $\calS$ be the basic or compound secret associated with $\Is$, and $S$ be the number of unique times in $\Is$. The mechanism $\calA(X) = X + G = Z$, where $G \sim \calN(\mathbf{0}, \Sigmag)$, then satisfies $(\varepsilon, \lambda)$-Conditional Inferential Privacy $(\Spairs, r, \Theta)$, where 
	\begin{align}
		\varepsilon
		&\leq \frac{\lambda}{2} S r^2 \Big(  \frac{1 }{\sigma_s^2} + \alpha^*  \Big) 
		\label{eqn: priv bound}
	\end{align}
	\text{ }\vspace{1mm}\\
	where $\sigma_s^2$ is the variance of each $G_i \in \Gs$ (diagonal entries of $\Sigmag_{ss}$) and $\alpha^*$ is the maximum eigenvalue of $\Sigmaeff = \big(\Sigma_{us} \Sigma_{ss}^{-1}\big)^\intercal \big( \Sigma_{u | s} + \Sigma_{uu}^{(g)} \big)^{-1} \big(\Sigma_{us} \Sigma_{ss}^{-1}\big)$. 
\end{theorem}

The above bound is tight for basic secrets ($S = 1$). The two terms of Equation \eqref{eqn: priv bound} represent the direct $(\frac{1}{\sigma_s^2})$ and inferential $(\alpha^*)$ loss terms of Equation \eqref{eqn:two terms}. We assume that each diagonal entry of $\Sigmag_{ss}$ equals some $\sigma_s^2$, so that each $X_i \in \Xs$ experiences identical direct privacy loss, which is optimal under utility constraints. 

The above bound composes gracefully when multiple traces of an individual are released. 

\begin{corollary}\emph{Graceful Composition in Time}
\label{cor: composition}
	Suppose a user releases two traces $X$ and $\hat{X}$ with additive noise $G \sim \calN(\mathbf{0}, \Sigmag)$ and $\hat{G} \sim \calN(\mathbf{0}, \hat{\Sigma}^{(g)})$, respectively. Then basic or compound secret $\Xs$ of $X$ enjoys $(\bar{\varepsilon}, \lambda)$-CIP, where 
	\begin{align*}
		\bar{\varepsilon} \leq \frac{\lambda}{2} S r^2 \Big(  \frac{1 }{\sigma_s^2} + \bar{\alpha}^*  \Big) 
	\end{align*}
	and where $\bar{\alpha}^*$ is the maximum eigenvalue of $\bar{\Sigma}_{\text{eff}} = \big(\Sigma_{us} \Sigma_{ss}^{-1}\big)^\intercal \big( \Sigma_{u | s} + \bar{\Sigma}_{uu}^{(g)} \big)^{-1} \big(\Sigma_{us} \Sigma_{ss}^{-1}\big)$. $\Sigma$ is the covariance matrix of the joint distribution on $X, \hat{X}$ and 
	\begin{align*}
	\bar{\Sigma}^{(g)} =
		\begin{bmatrix}
			 \Sigmag & 0 \\
			 0 &  \hat{\Sigma}^{(g)} \ .
		\end{bmatrix}
	\end{align*}
\end{corollary}
\text{ } \vspace{2mm} \\
This bound is identical to that of Theorem \ref{thm:GP bound}, only using the joint distribution over $X$, $\hat{X}$ and $G, \hat{G}$. This provides some insight to the fact that, unlike DP, even parallel composition guarantees are not automatic. Composition depends on the conditional prior. In the GP setting, if the chosen kernel function decays over time, we can expect composition to have minimal effects on privacy for traces separated by long durations. 

To reduce the upper bound of Theorem \ref{thm:GP bound}, we optimize the correlation (off-diagonal) of $\Sigmag$ to minimize $\alpha^*$, and optimize its variance (diagonal) to balance a noise budget between lowering inferential ($\alpha^*$) and direct ($\frac{1}{\sigma_s^2}$) loss.

\section{Optimized Privacy Mechanisms}
\label{sec: algorithms}

Theorem \ref{thm:GP bound} characterizes the privacy loss for GP conditional priors. We next show how to use this Theorem to design mechanisms that can strategically reduce CIP loss given a utility constraint. We measure `utility loss' as the total mean squared error (MSE) between the released ($Z$) and true ($X$) traces: $\text{MSE}(\Sigmag) = \sum_{i=1}^n \E[Z_i - X_i] = \trace(\Sigmag)$. We bound the utility loss by $\trace(\Sigmag) \leq n o_t$, where $o_t$ is the average per-point utility loss.

It can be shown that optimizing the privacy loss under this utility constraint can be described by a semidefinite program (SDP) (formalization/derivation of SDPs in Appendix \ref{apx: algorithmns}). For a given trace $X$, define its covariance matrix $\Sigma$ using the the kernel of the GP conditional prior $\Sigma_{ij} = k(i,j)$. Then pass $\Sigma$, the secret set $\Is$, and the utility constraint $o_t$ to our first program, $\text{SDP}_\text{A}$, which returns noise covariance $\Sigmag$. This defines an additive noise mechanism $G \sim \calN(0, \Sigmag)$ that minimizes CIP loss to $\Is$. 
\begin{align*}
	\Sigmag = \text{SDP}_\text{A}(\Sigma, \Is, o_t)
\end{align*} 
We can thus use a SDP to minimize the CIP loss to any single compound or basic secret. However, a trace may contain multiple locations or combinations thereof that one wishes to protect. It remains to produce a single mechanism $\Sigmag$ that bounds the CIP loss to multiple basic and/or compound secrets in a single trace. 

For this we propose $\text{SDP}_\text{B}$, which uses the fact that if ${\Sigmag}' \succ \Sigmag$ it will have lower CIP loss (see Appendix \ref{apx: SDP B}). $\text{SDP}_\text{B}$ takes in a set of covariance matrices $\calF = \{\Sigmag_1, \dots, \Sigmag_m\}$, each designed to minimize CIP loss for a single compound or basic secret $\Is_i$. It then returns a single covariance matrix $\Sigmag \succeq \Sigmag_i, i \in [m]$ that maintains the privacy guarantee each $\Sigmag_i$ offered its corresponding $\Is_i$, while minimizing utility loss. 

\begin{figure*}[h]
    \centering
    \begin{subfigure}{.24\linewidth}
        \centering 
        \captionsetup{justification=centering}
        \includegraphics[width = 1\linewidth]{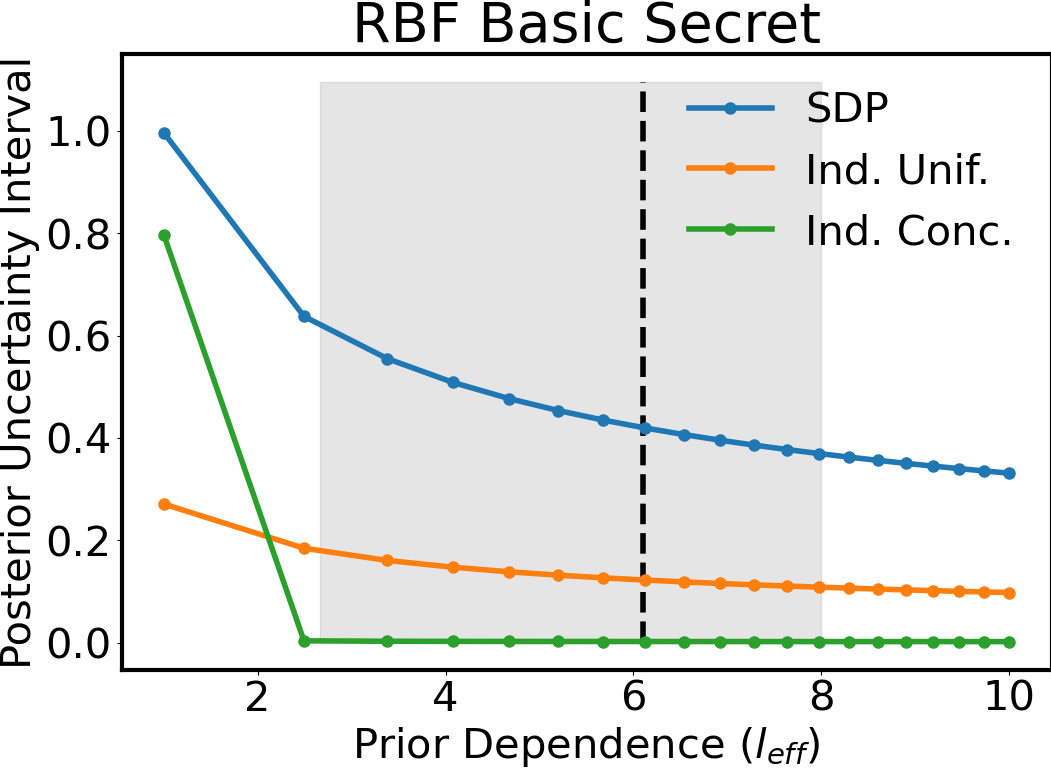}
        \caption{}
        \label{fig: RBF basic}
    \end{subfigure}
    \begin{subfigure}{.24\linewidth}
        \centering 
        \captionsetup{justification=centering}
        \includegraphics[width = 0.95\linewidth]{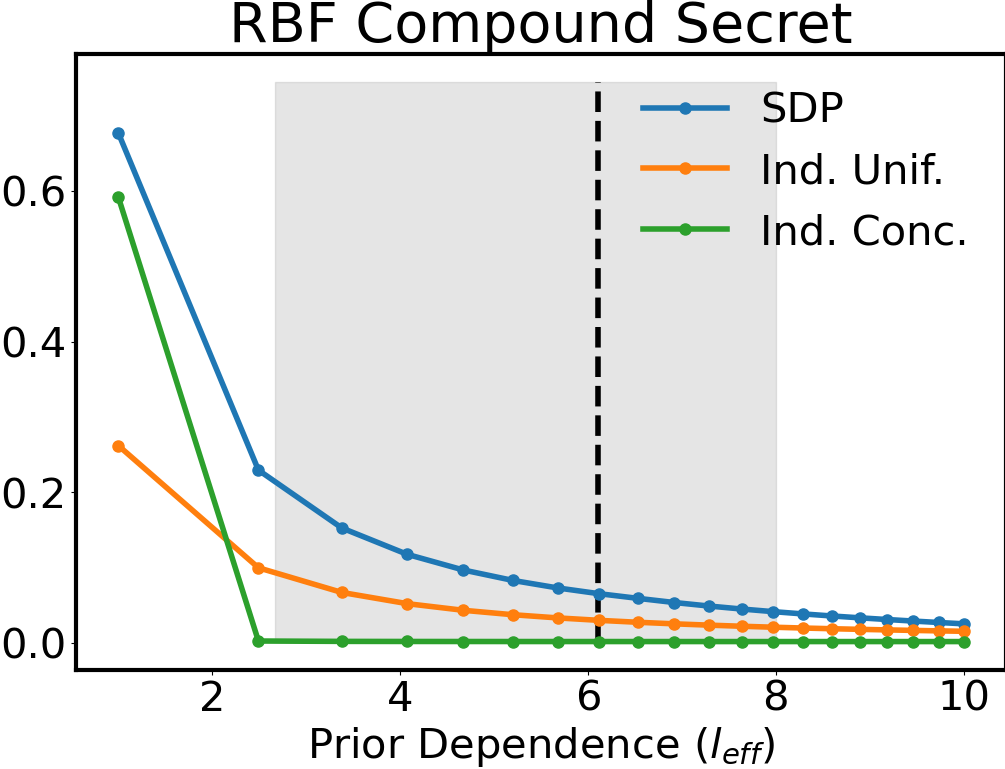}
        \caption{}
        \label{fig: RBF compound}
    \end{subfigure}
    \begin{subfigure}{.24\linewidth}
        \centering 
        \captionsetup{justification=centering}
        \includegraphics[width = 0.95\linewidth]{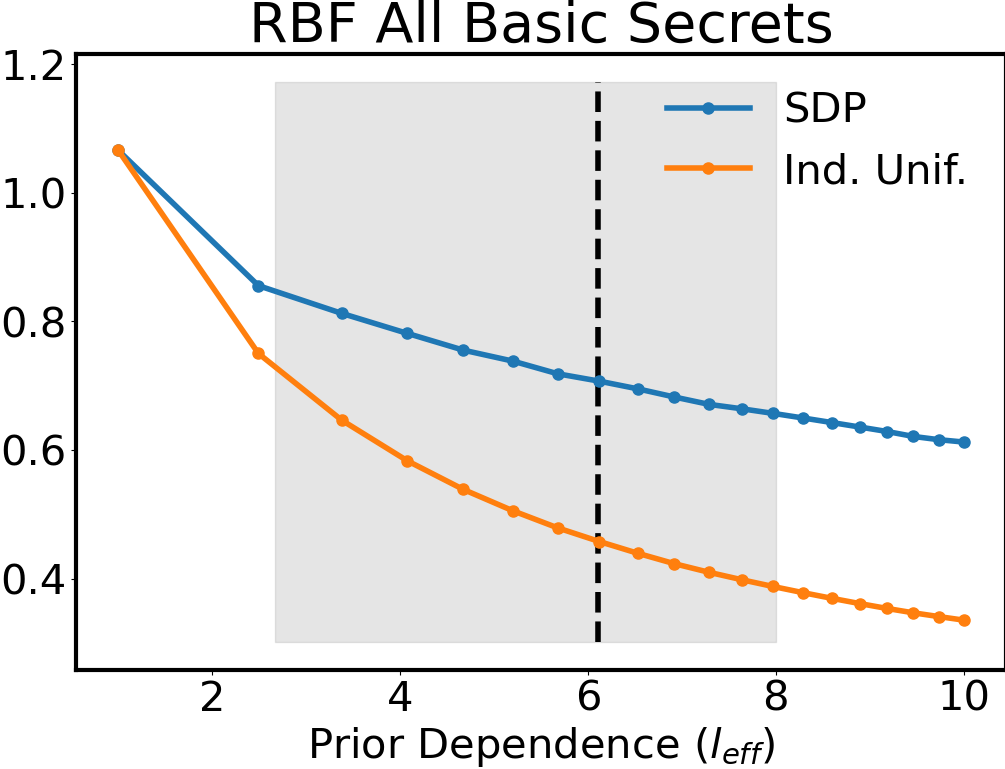}
        \caption{}
        \label{fig: RBF all}
    \end{subfigure}
    \begin{subfigure}{.24\linewidth}
        \centering 
        \captionsetup{justification=centering}
        \includegraphics[width = 0.95\linewidth]{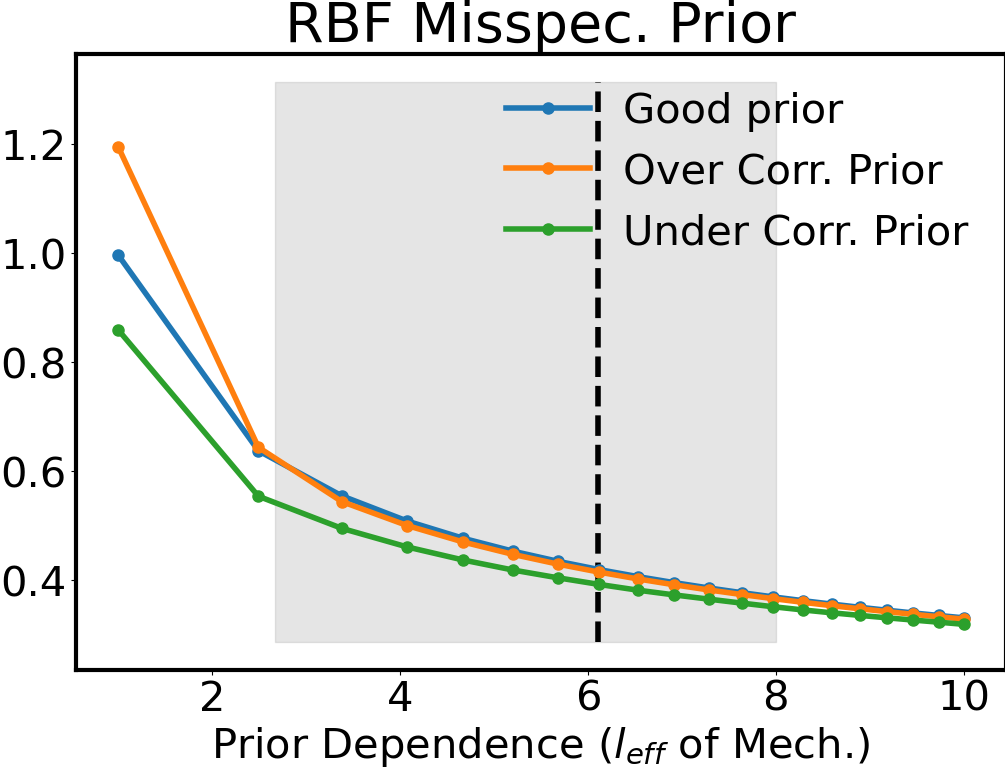}
        \caption{}
        \label{fig: RBF misspec}
    \end{subfigure}
    \begin{subfigure}{.24\linewidth}
        \centering 
        \captionsetup{justification=centering}
        \includegraphics[width = 1\linewidth]{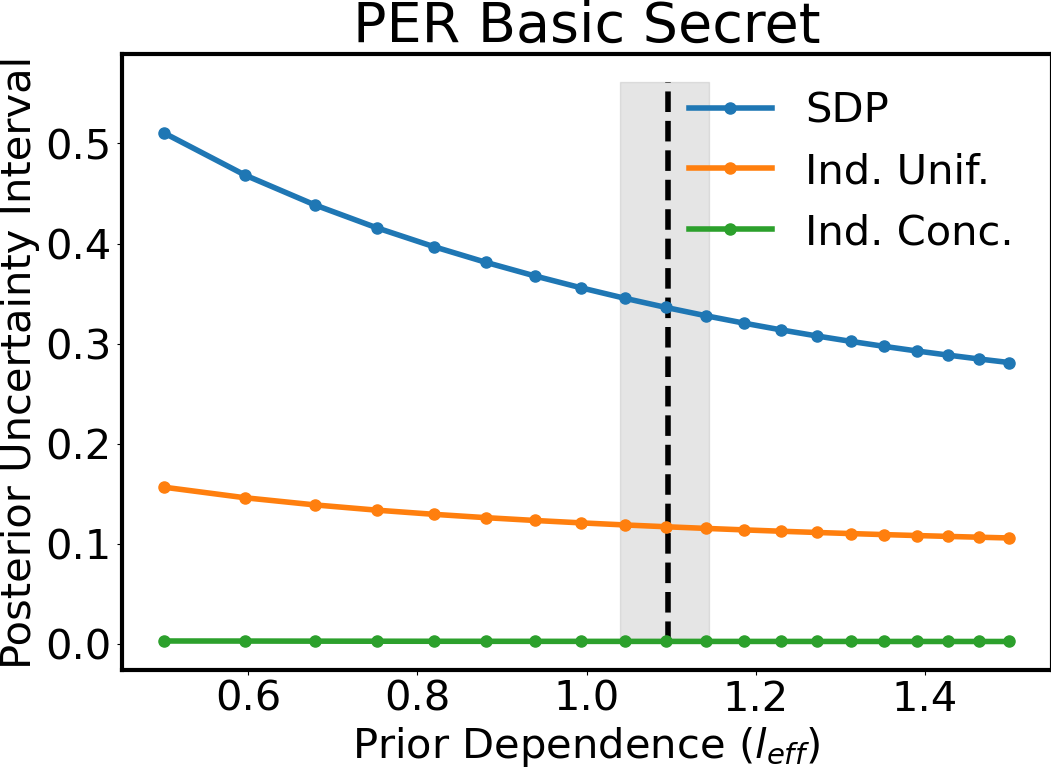}
        \caption{}
        \label{fig: PER basic}
    \end{subfigure}
    \begin{subfigure}{.24\linewidth}
        \centering 
        \captionsetup{justification=centering}
        \includegraphics[width = 0.95\linewidth]{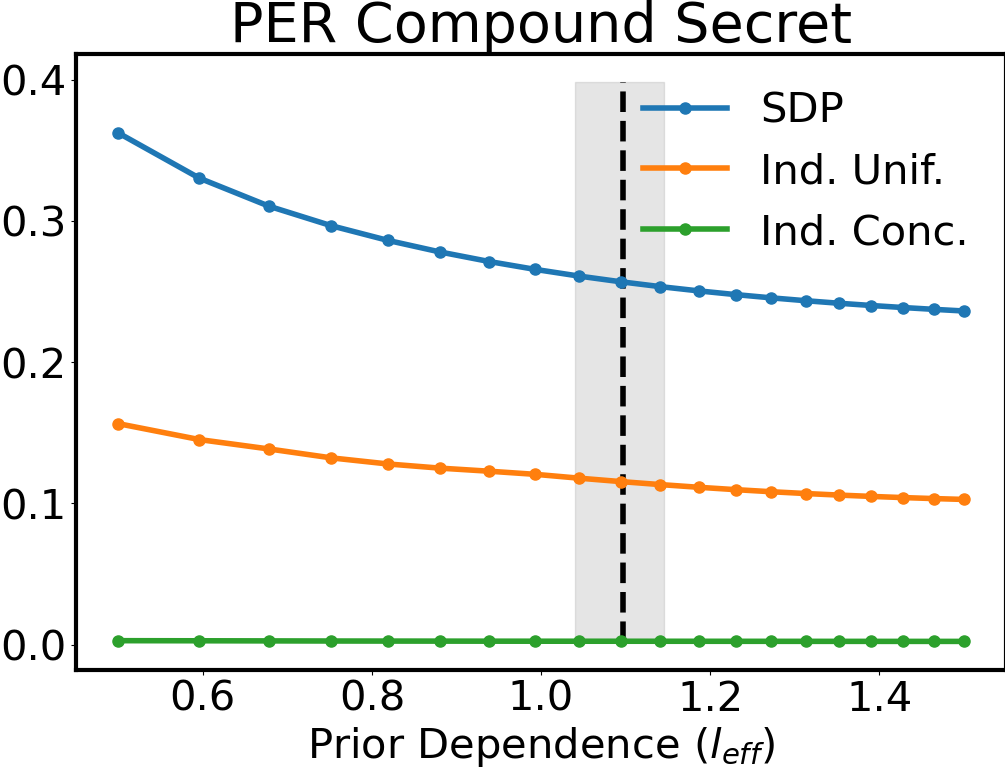}
        \caption{}
        \label{fig: PER compound}
    \end{subfigure}
    \begin{subfigure}{.24\linewidth}
        \centering 
        \captionsetup{justification=centering}
        \includegraphics[width = 0.95\linewidth]{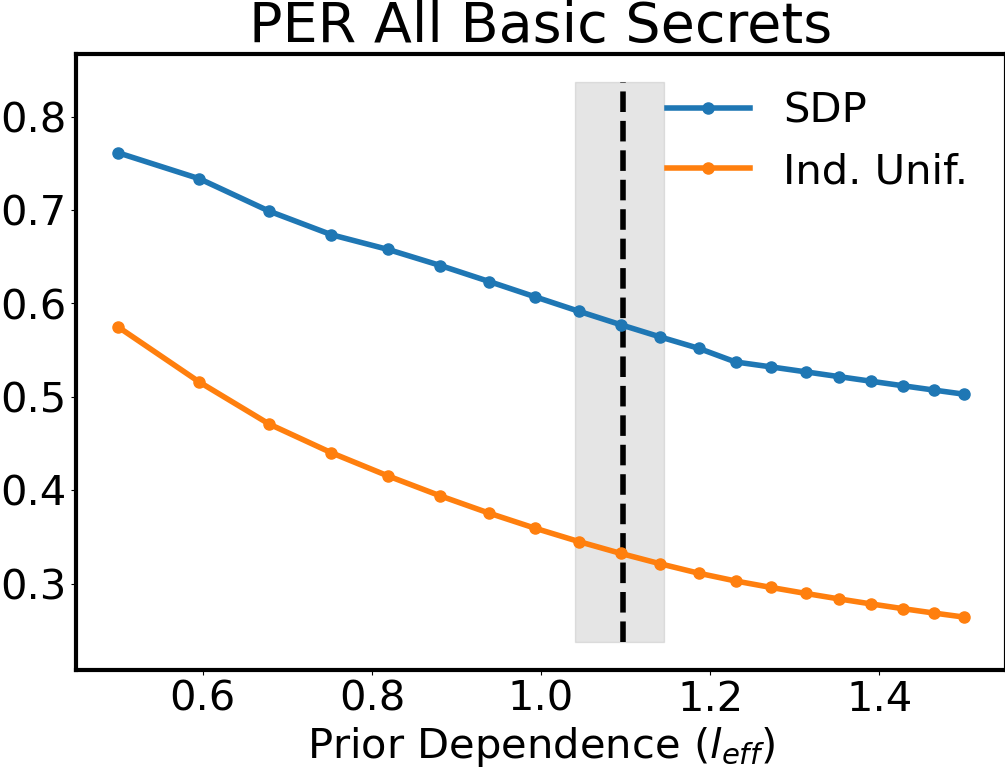}
        \caption{}
        \label{fig: PER all}
    \end{subfigure}
    \begin{subfigure}{.24\linewidth}
        \centering 
        \captionsetup{justification=centering}
        \includegraphics[width = 0.95\linewidth]{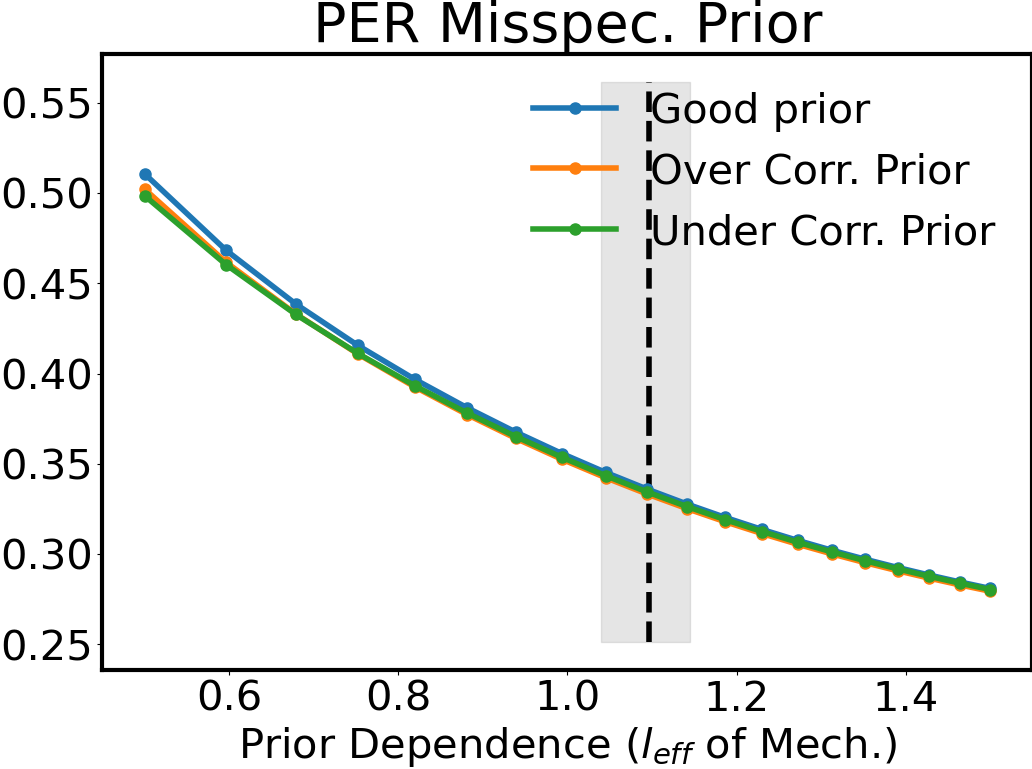}
        \caption{}
        \label{fig: PER misspec}
    \end{subfigure}
    \caption{$^1$Posterior uncertainty interval (higher$=$better privacy) on $\Xs$ of a GP Bayesian adversary. A larger $\leff$ corresponds to greater inter-dependence and reduces posterior uncertainty. The gray interval depicts the middle 50\% of the MLE $\leff$ among traces in each dataset, and the black dotted line the median $\leff$. \textbf{(a)}$\rightarrow$\textbf{(c)}, \textbf{(e)}$\rightarrow$\textbf{(g)} show SDP mechanisms (blue) maintaining relatively high uncertainty compared to two GI (Approach C) baselines of equal utility (MSE). \textbf{(d)}, \textbf{(h)} show the (minor) change in posterior uncertainty when the prior covariance $\Sigma$  used in $\text{SDP}_{\text{A}}$ is misspecified: when it is identical to the true covariance $\Sigma^*$ known to the adversary (blue), is more correlated (orange), or is less correlated (green).
    }
    \label{fig: experiments}
\end{figure*} 

\begin{algorithm}
		\SetAlgoLined
		\KwInput{$\Is_1, \dots, \Is_m, o_t, \Sigma$}
		\KwOutput{$\Sigmag$}
			\vskip 1mm
			$\calF = \emptyset$\;
			
			\For{$i \in [m]$}
			{
				$\Sigmag_i =$ $\text{SDP}_\text{A}(\Sigma, \Is_i, o_t)$\; 
				
				$\calF = \calF \cup \Sigmag_i$\;
			}
			\vskip 1mm
			$\Sigmag= \text{SDP}_{\text{B}}(\calF)$\;
			
			\Return $\Sigmag$\;
		\caption{Multiple Secrets}
\label{alg: Multiple Secrets}
\end{algorithm}
In our experiments, we use Algorithm \ref{alg: Multiple Secrets} to design a single mechanism that protects all locations in the trace --- all basic secrets --- while minimizing utility loss.

\section{Experiments}
\label{sec: experiments}
Here, we aim to empirically answer: \textbf{1)} 
Do our SDP mechanisms maintain high posterior uncertainty of sensitive locations? How do they compare to Approach C baselines of equal MSE? 
\textbf{2)} How robust is the $\text{SDP}_\text{A}$ mechanism when the prior covariance $\Sigma$ is misspecified? 

\paragraph{Methods} To answer these questions, we look at the range of conditional prior classes that fit real-world data. For location trace data, we use the GeoLife GPS Trajectories dataset \citep{geolife} containing 10k human mobility traces after preprocessing (see Appendix \ref{apx: experiments} for details). We also consider the privacy risk of room temperature data \citep{home_monitoring}, using the SML2010 dataset \citep{sml2010}, which contains approximately 40 days of room temperature data sampled every 15 minutes. 

For the location data, having observed that the correlation between latitude and longitude is low ($ \approx 0.06$) we treat each dimension as independent. By way of Corollary \ref{cor: independence}, this allows us to bound privacy loss and design mechanisms for each dimension separately. Furthermore, having observed that each dimension fits nearly the same conditional prior, we treat our dataset of 10k 2-dimensional traces as a dataset of 20k 1-dimensional traces, where each trace represents one dimension of a 2d location trajectory.

We model the location trace data with a Radial Basis Function (RBF) kernel GP and the temperature series data with a periodic kernel GP:
\begin{align*}
	k_{\text{RBF}}(t_i, t_j) 
	&=  \sigma_x^2 \exp \Big( -\frac{(t_i - t_j)^2}{2 l^2} \Big) \\
	k_{\text{PER}}(t_i, t_j) 
	&=  \sigma_x^2 \exp \Big(  \frac{-2 \sin^2(\pi |t_i - t_j| / p)}{l^2} \Big)
\end{align*}
In both kernels, the intrinsic degree of dependence between points is captured by the lengthscale $l$. However, the fact that sampling rates vary significantly between traces means that traces with equal length scales can have very different degrees of correlation. To encapsulate both of these effects, we study the empirical distribution of \emph{effective} length scale of each trace
\begin{align*}
	l_{\text{eff},x} = \frac{l_x}{P}
	\quad
	l_{\text{eff},y} = \frac{l_y}{P}
\end{align*}
where $P$ is the trace's sampling period and $l_x,l_y$ are the its optimal length scales for each dimension. 

$l_{\text{eff},x},l_{\text{eff},y}$ tell us the average number of neighboring locations that are highly correlated, instead of time period. For instance, a given trace with an optimal $l_{\text{eff},x} = 8$ tells us that every eight neighboring location samples in the $x$ dimension have correlation $> 0.8$. The empirical distribution of effective length scales across all traces describes -- over a range of logging devices (sampling rates), users, and movement patterns -- how many neighboring points are highly correlated in location trace data. After this preprocessing, we are able to use the kernels that take indices (not time) as arguments: 
\begin{align*}
	\label{eqn: kernels}
	k_{\text{RBF}}(i, j) 
	&=  \exp \Big( -\frac{(i - j)^2}{2\leff^2} \Big) \\
	k_{\text{PER}}(i, j) 
	&=  \exp \Big(  \frac{-2 \sin^2(\pi |i - j| / p)}{\leff^2} \Big)
\end{align*}
See Appendix \ref{apx: experiments} for a more detailed discussion of how the empirical distribution of $\leff$ across traces is measured. 

To impart the range of realistic conditional priors the gray interval of each plot depicts the middle 50\% of the empirical $\leff$ among traces in each dataset. The dashed vertical line reports the median $\leff$.


Each figure increases the degree of dependence, $\leff$, used by the kernel to compute the prior covariance $\Sigma(\leff)$. $\Sigma(\leff)$ is then used in one of the SDP routines of Section \ref{sec: algorithms} to produce a mechanism $\Sigmag(\leff)$ that protects a basic secret ($\text{SDP}_\text{A}$), a compound secret ($\text{SDP}_\text{A}$), or the union of all basic secrets (Multiple Secrets). We then observe the 68\% confidence interval of the Gaussian posterior on sensitive points $\Xs$ (blue line). This is the $2\sigma$ uncertainty of a Bayesian adversary with a GP prior represented by $\Sigma(\leff)$ (see Appendix \ref{apx: experiments} for how this is computed). As $\leff$ increases, their posterior uncertainty will reduce. Our aim is to mitigate this as much as possible with the given utility constraint. For scale, recall that prior variance $\textbf{diag}(\Sigma)$ is normalized to one. In the case of all basic secrets, we report the average posterior uncertainty over locations. 

We compare the SDP mechanisms with two mechanisms using the logic of Approach C (all three of equal MSE utility loss): \emph{independent/uniform} and \emph{independent/concentrated}. The uniform approach adds independent Gaussian noise evenly along the whole trace regardless of $\Is$, $\Sigmag = o_tI$. The concentrated approach allocates the entire noise budget to the sensitive set $\Is$. 
\paragraph{Results}
For our first question, see \textbf{Figures \ref{fig: RBF basic}$\rightarrow$\ref{fig: RBF all}, \ref{fig: PER basic}$\rightarrow$\ref{fig: PER all}}. For both location and temperature data, our SDP mechanisms maintain higher posterior uncertainty than the baselines with identical utility cost for a single basic secret, a compound secret, and all basic secrets. By actively considering the conditional prior class parametrized by $\Sigma$, the SDP mechanisms can strategize to both correlate noise samples and concentrate noise power such that posterior inference is thwarted at the sensitive set $\Is$. For an intuitive illustration of the chosen $\Sigmag$'s, see Appendix \ref{apx: juxtaposition}. 

To answer our second question, see \textbf{Figures \ref{fig: RBF misspec}} and \textbf{\ref{fig: PER misspec}}. When the prior covariance $\Sigma$ does not represent the true data distribution known to the adversary, a smaller posterior uncertainty may be achieved. The orange line indicates the uncertainty interval of an adversary who knows the data is \emph{less} correlated than we believe i.e. the true $\Sigma^* = \Sigma(0.5 \leff)$. The blue line represents an adversary who knows the data is \emph{more} correlated than we believe i.e. the true $\Sigma^* = \Sigma(1.5 \leff)$. Both plots confirm the robustness of our privacy guarantees stated by Theorem \ref{thm: prior misspecification}. Particularly around the median $\leff$ we see that the change in posterior uncertainty with this change in prior is indeed marginal. 

\section{Discussion}
\paragraph{Related Work}
Few works have proposed solutions to the \emph{local} guarantee when releasing individual traces. A mechanism offered in \cite{synthesizing_plausible_deniability} releases synthesized traces satisfying the notion of \emph{plausible deniability} \citep{plausible_deniability}, but this is distinctly different from providing a radius of privacy to sensitive locations. Meanwhile, the frameworks proposed in \cite{temporal} and \cite{priste} nicely characterize the risk of inference in location traces, but use only first-order Markov models of correlation between points, do not offer a radius of indistinguishability as in this work, and are not suited to continuous-valued spatiotemporal traces.

Perhaps more technically similar to this work, \cite{song_pufferfish_2017} provide a general mechanism that applies to any Pufferfish framework, as well as a more computationally efficient mechanism that applies when the joint distribution of an individual's features can be described by a graphical model. The first is too computationally intensive. The second is for discrete settings, and cannot accommodate spatiotemporal effects.

\paragraph{Conclusion}
This work proposes a framework for both identifying and quantifying the \emph{inferential} privacy risk for highly dependent sequences of spatiotemporal data. As a starting point, we have provided a simple bound on the privacy loss for Gaussian process priors, and an SDP-based privacy mechanism for minimizing this bound without destroying utility. We hope to extend this work to other data domains with different conditional priors, and different sets of secrets.

\subsubsection*{Acknowledgements}
KC and CM would like to thank ONR under N00014-20-1-2334 and UC Lab Fees under LFR 18-548554  for research support. We would also like to thank our reviewers for their insightful feedback. 

\clearpage

\bibliography{refs.bib}
\bibliographystyle{icml2020}

\newpage
\onecolumn

\section{Appendix}
For documented code demonstrating our SDP mechanisms used to generate the plots of \textbf{Figure \ref{fig: experiments}} please visit our repo: \url{https://github.com/casey-meehan/location_trace_privacy} 

The following sections will include proofs of results, derivations of algorithms, and explanations of experimental procedures. 
\subsection{Illustrations}
\label{apx: Illustrations}

\subsubsection{NYC Mayoral Staff Member Location Trace}

\begin{figure*}[h]
	\centering
	\begin{subfigure}[b]{.32\textwidth}
		\centering
		\includegraphics[width = \linewidth]{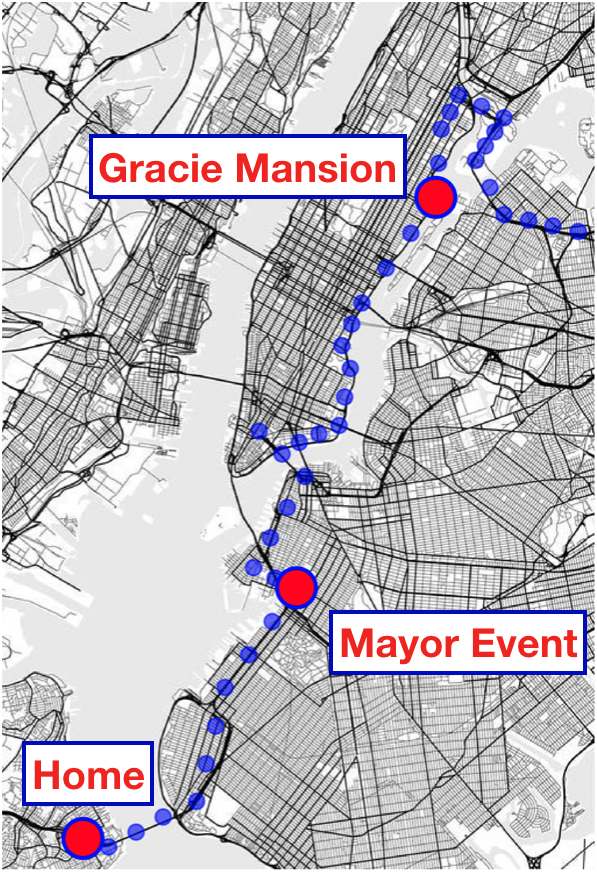}
		\caption{}
		\label{fig:nyc_trace}
	\end{subfigure}
	\begin{subfigure}[b]{.32\textwidth}
		\centering
		\includegraphics[width = \linewidth]{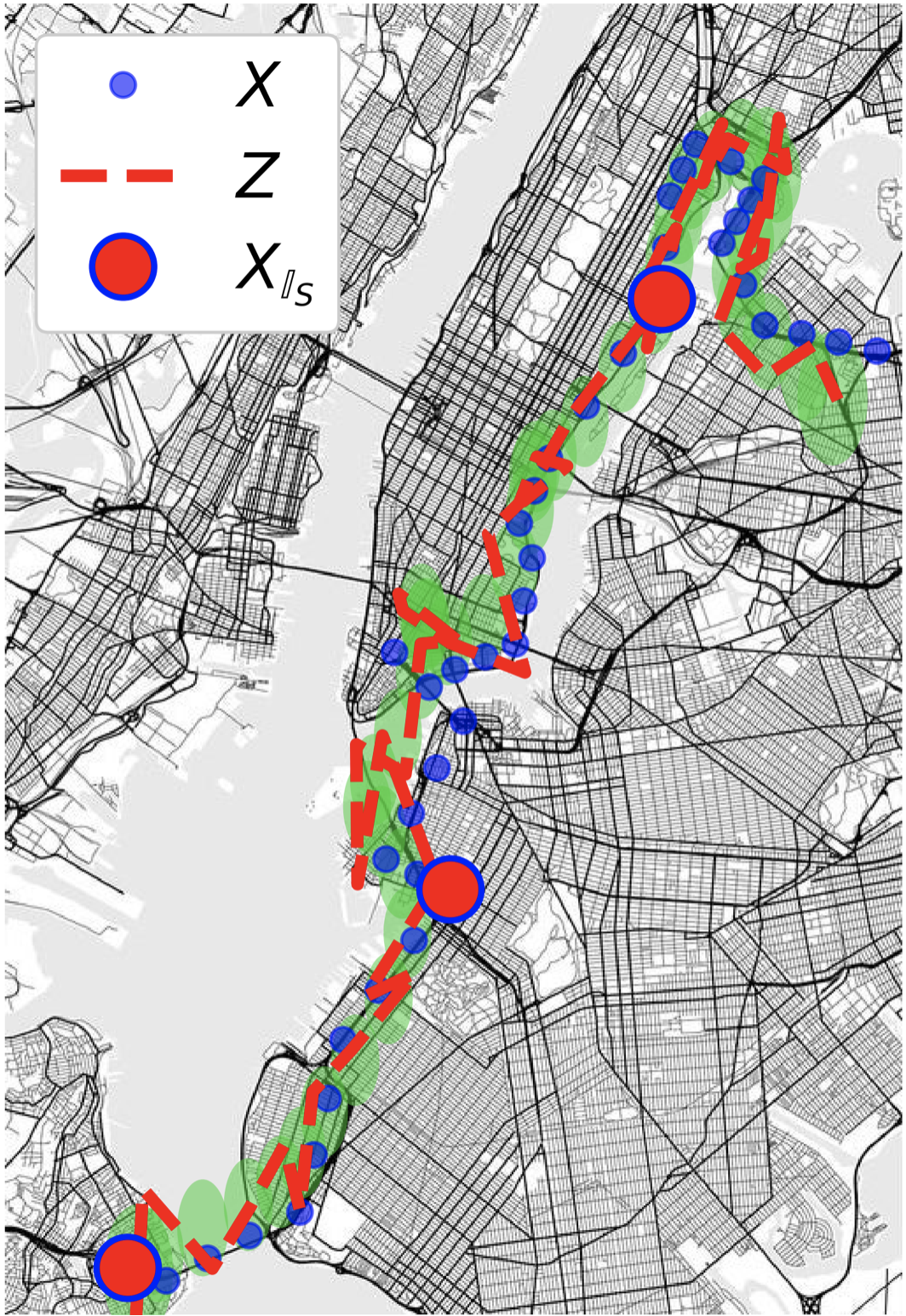}
		\caption{}
		\label{fig:nyc_trace_iso}
	\end{subfigure}
	\begin{subfigure}[b]{.32\textwidth}
		\centering
		\includegraphics[width = \linewidth]{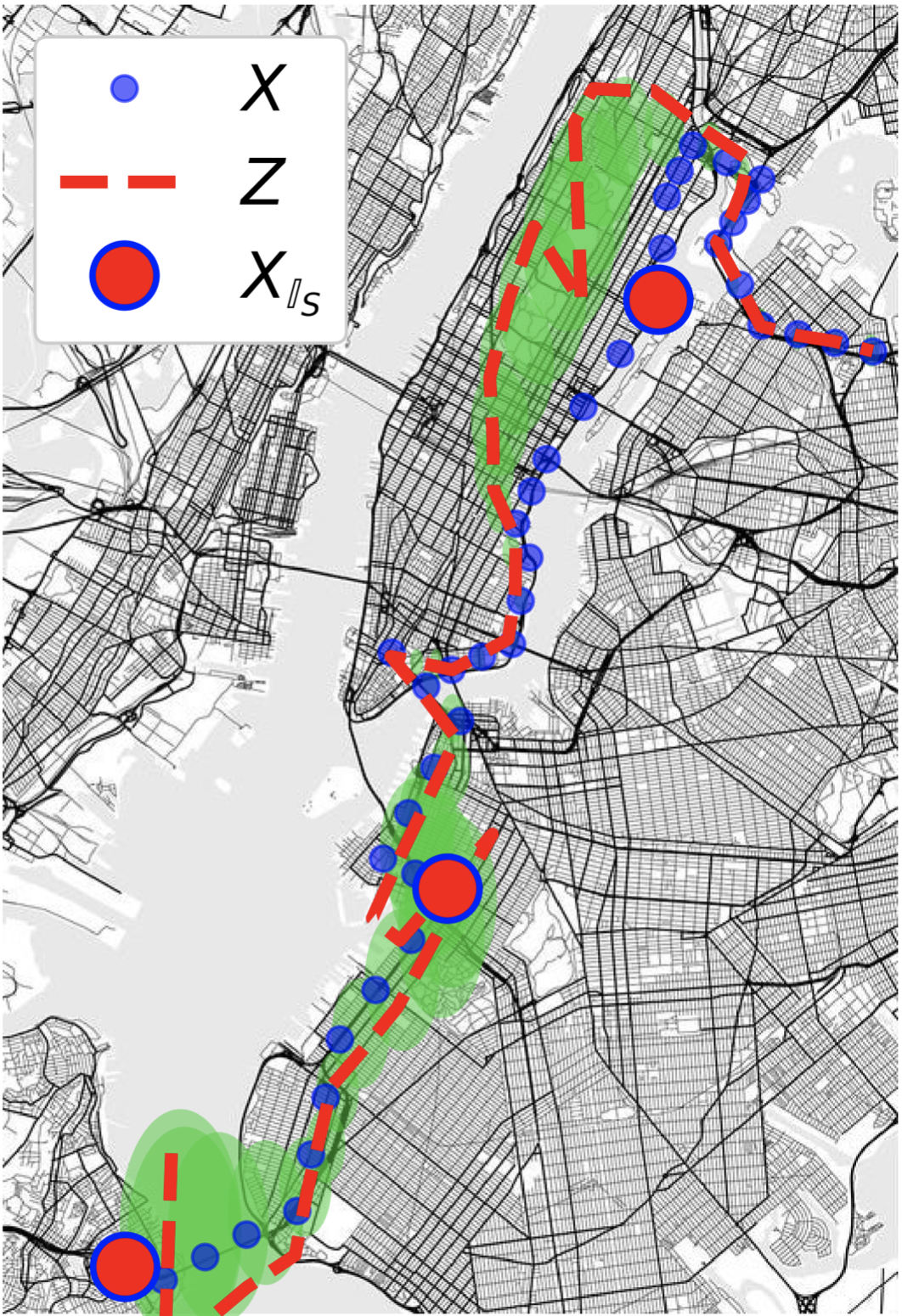}
		\caption{}
		\label{fig:nyc_trace_opt}
	\end{subfigure}
	\caption{Example of sensitive location trace of NYC mayoral staff member exposed by \citep{nyt}. (b) and (c) depict the posterior uncertainty (green) $P_{\calA,\calP}(X_i | Z)$ for each 2d location. (a) depicts three sensitive times (red with blue outline): Gracie Mansion (Mayor's home), an event on Staten Island that the mayor attended, and finally the staff member's home on long island. (b) provides an example of Approach C: adding independent Gaussian noise to each location (red dotted line). A GP posterior still maintains high confidence within a small radius along the trace, including at the sensitive times. (c) provides an example of the optimized noise of Multiple Secrets of identical aggregate MSE as (b). By focusing \textit{correlated} noise around the three sensitive times, there is high uncertainty at sensitive times and high confidence elsewhere.}
	\label{fig:nyc_example}
\end{figure*}

\subsubsection{Juxtaposition of Mechanisms' Covariance Matrices}
\label{apx: juxtaposition}
The following figures aim to illustrate the difference between the covariance matrices used in the experimental baselines (indep./uniform and indep./concentrated) and those chosen by our SDP algorithms for both the RBF and periodic prior. Note that here we presume the different dimensions of location to be independent and --- by Corollary \ref{cor: independence} --- are able to treat a 2d location trace as two 1d traces. As such, the following examples are demonstrating mechanism covariance matrices and additive noise samples used for either a single dimension of location data (for RBF kernel) or for the one dimension of temperature data (for periodic kernel). 

The first figure \textbf{(a)} shows the covariance of the Approach C baselines used in the experiments. The second figure \textbf{(b)} shows the covariance of our SDP mechanisms for the RBF kernel used on location data. The third figure \textbf{(c)} shows the covariance of our SDP mechanisms for the periodic kernel used for temperature data. 

In each figure the covariance matrix is depicted as a heat map with warmer colors indicating higher values (normalized to largest and smallest value in the covariance matrix). The drawn noise samples $G$ are plotted against their time index. So, the sequence of plotted $(x,y)$ values is $\big[(1, G_1), (2, G_2), \dots, (n, G_n)\big]$, where $n = 50$ for the RBF case and $n = 48$ for the periodic case. 

\begin{figure*}[h]
	\centering
	\begin{subfigure}[b]{1\textwidth}
		\centering
		\includegraphics[width = \linewidth]{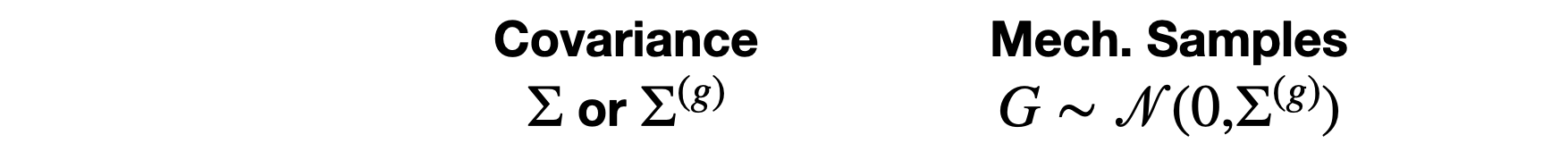}
	\end{subfigure}
	\begin{subfigure}[b]{1\textwidth}
		\centering
		\includegraphics[width = \linewidth]{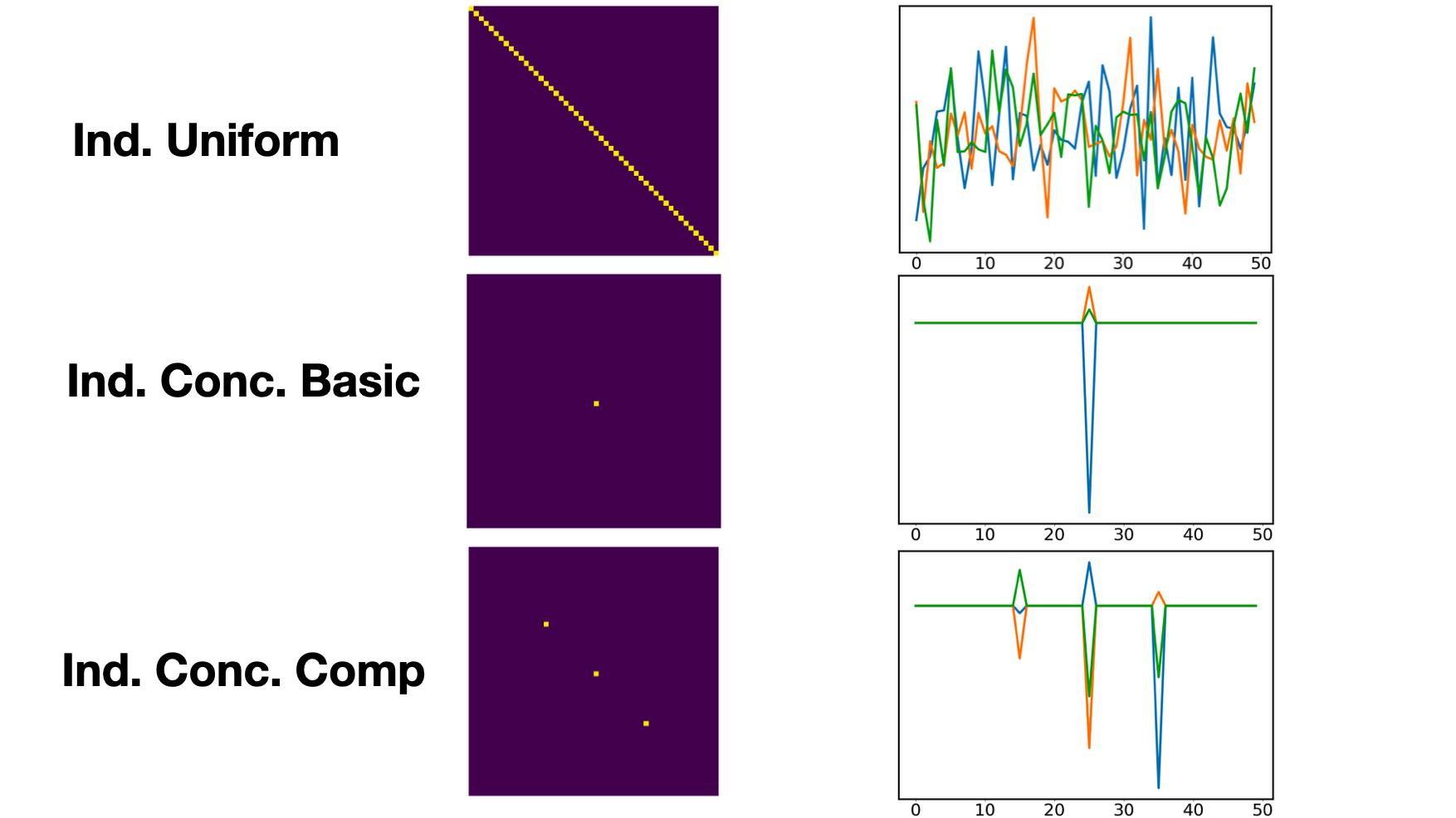}
		\caption{Covariance matrices and mechanism samples for the baselines used in experiments. 
		\vspace{2mm}\\
		The first figure demonstrates the uniform approach that distributes the independent Gaussian noise budget along the entire trace, regardless of $\Is$. 
		\vspace{2mm}\\
		The second and third show the concentrated approach that allocates the entire noise budget to only the sensitive locations in $\Is$: first for a basic secret (one location) and then for a compound secret of 3 evenly spaced locations.} 
		\label{fig: cov table baselines}
	\end{subfigure}
\end{figure*}

\begin{figure*}[h] \ContinuedFloat
	\begin{subfigure}[b]{1\textwidth}
		\centering
		\includegraphics[width = \linewidth]{./images/cov_table_header.png}
	\end{subfigure}
	\begin{subfigure}[b]{1\textwidth}
		\centering
		\includegraphics[width = \linewidth]{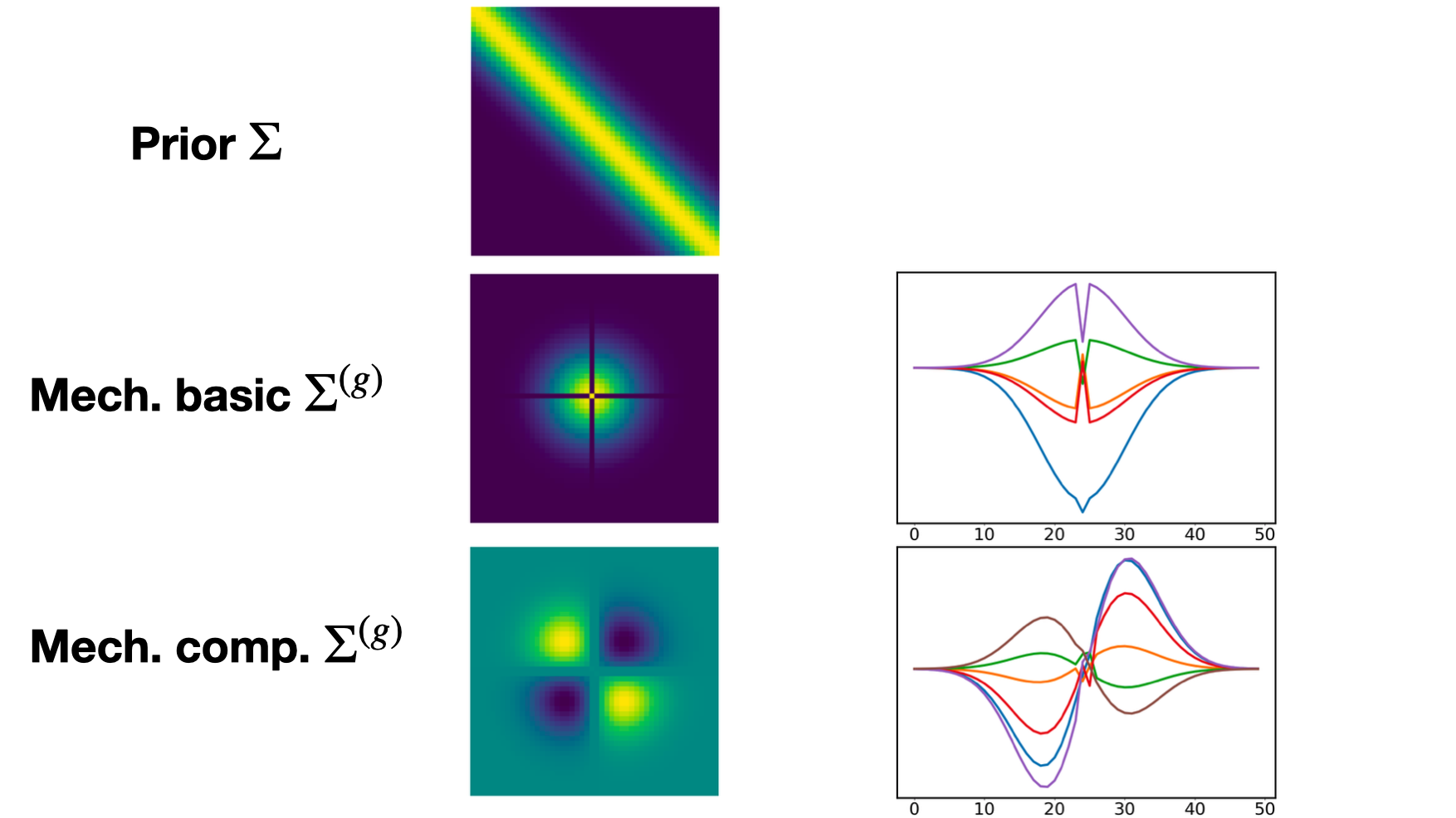}
	\end{subfigure}
	\begin{subfigure}[b]{1\textwidth}
		\centering
		\includegraphics[width = \linewidth]{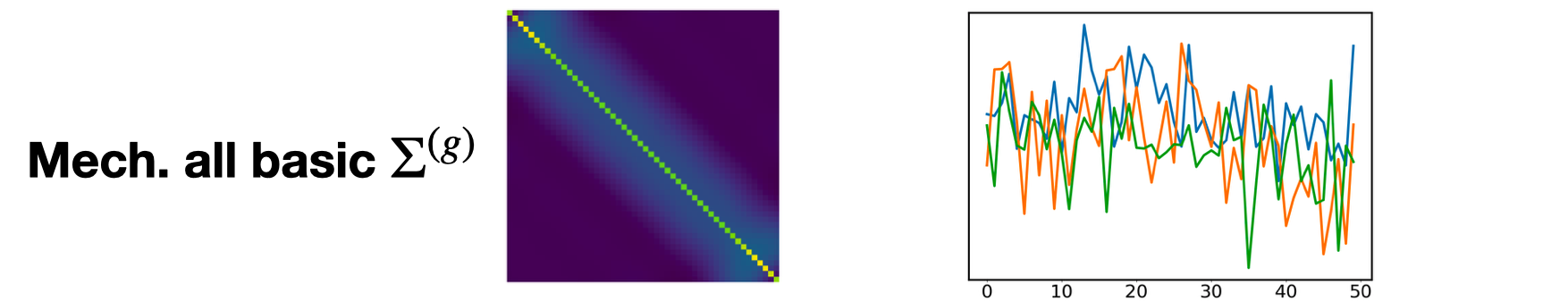}
		\caption{
			Covariance matrices and mechanism samples for the median RBF prior ($\leff \approx 6$). 
			\vspace{2mm} \\
			The first noise mechanism (Mech. basic) demonstrates the covariance matrix chosen by $\text{SDP}_\text{A}$ for a basic secret of a single location $X_i$ in the middle of the trace. The uncorrelated dot in the middle of the covariance matrix, $\Sigmag_{ii}$, represents the independent noise $G_i$ added at the sensitive location to mitigate \emph{direct} loss. To mitigate \emph{inferential} loss, the SDP optimizes the remainder of the matrix to be positively correlated with maximum variance allocated to locations near $X_i$ in time. This thwarts GP inference of the true location at time $t_i$. 
			\vspace{2mm} \\
			The second mechanism (Mech. comp.) depicts the covariance chosen by $\text{SDP}_\text{A}$ to protect a compound secret of two adjacent locations in the trace (visible as the uncorrelated `$+$' through the middle consuming 2 rows/columns). Recall that a compound secret ought to protect directional information: \emph{did the user visit B first and then A, or A and then B?} That is precisely what this mechanism does by randomizing the angle of approach to the two locations in the middle with positively and negatively correlated noise. Also note that the SDP does not allocate a large share of noise budget to the actual locations themselves. This highlights the fact that protecting a compound secret does not protect its constituent basic secrets.
			\vspace{2mm} \\
			The third and final mechanism (Mech. all basic) is the noise covariance chosen by $\text{SDP}_\text{B}$ in the Multiple Secrets algorithm. To protect all basic secrets with a utility constraint, the SDP converges to a mechanism that looks similar to the uniform baseline. However, this mechanism adds a subtle degree of off-diagonal correlation along with greater noise power towards the beginning and end of the trace. The off-diagonal correlation is noticeable when the samples are compared to those of the uniform baseline in the previous figure. While this change appears to be minor, it makes a significant change in the posterior confidence of a GP adversary (as seen in \textbf{Figure \ref{fig: RBF all}}). 
			}
		\label{fig: cov table rbf}
	\end{subfigure}
\end{figure*}

\begin{figure*}[h] \ContinuedFloat
	\begin{subfigure}[b]{1\textwidth}
		\centering
		\includegraphics[width = \linewidth]{./images/cov_table_header.png}
	\end{subfigure}
	\begin{subfigure}[b]{1\textwidth}
		\centering
		\includegraphics[width = \linewidth]{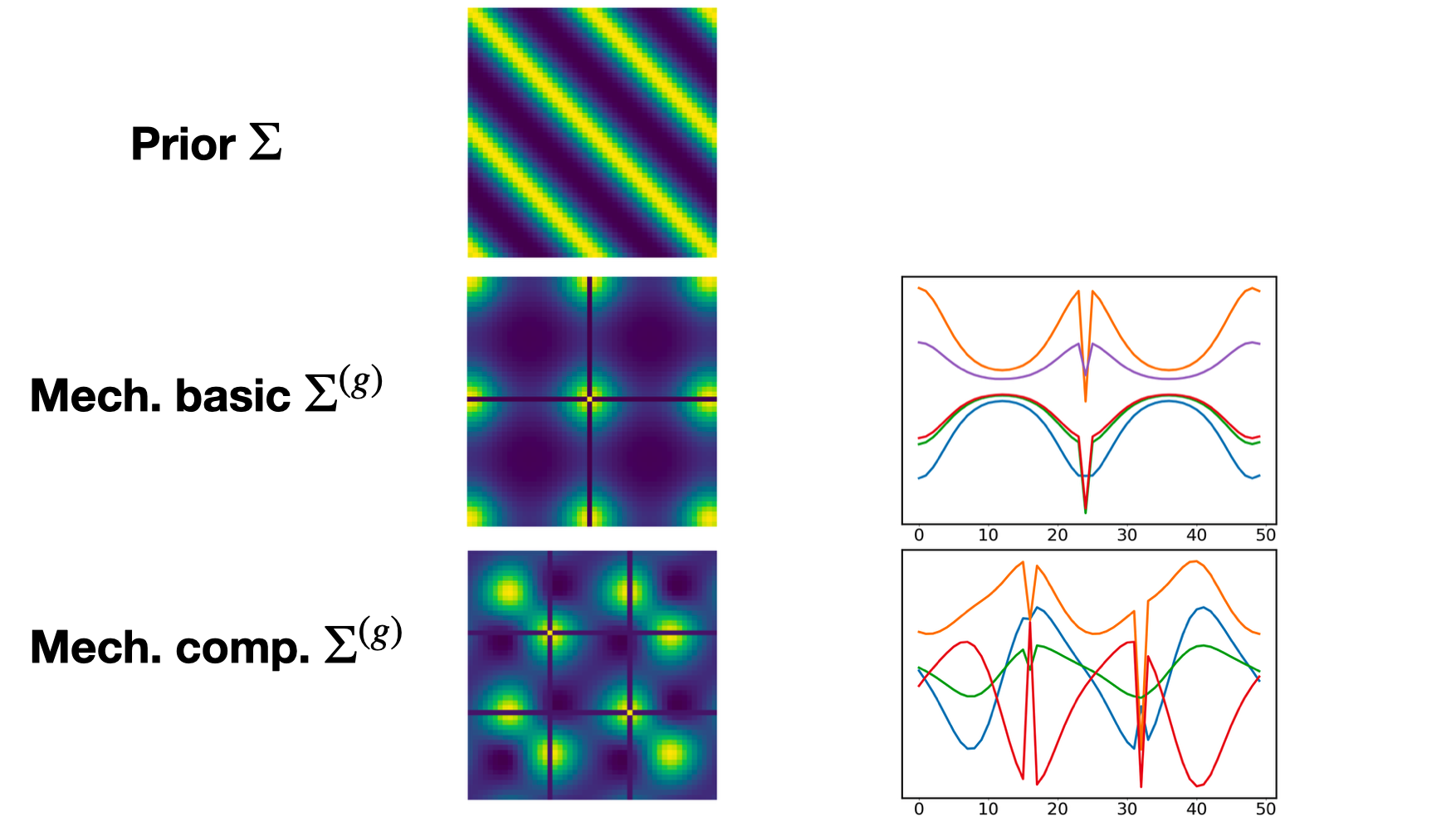}
	\end{subfigure}
	\begin{subfigure}[b]{1\textwidth}
		\centering
		\includegraphics[width = \linewidth]{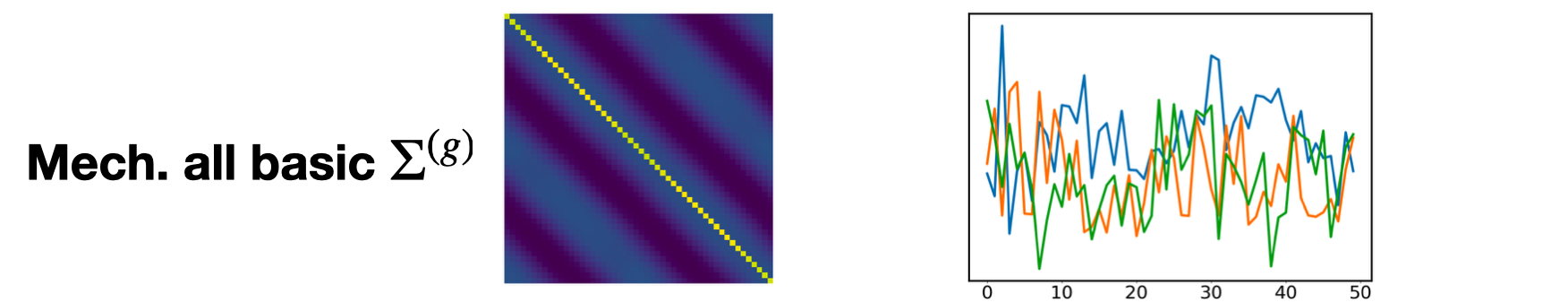}
		\caption{
			Covariance matrices and mechanism samples for the median periodic prior ($\leff \approx 1.1$), and a period of half the trace length. 
			\vspace{2mm} \\
			The first noise mechanism (Mech. Basic) shows the covariance chosen by $\text{SDP}_\text{A}$ to protect a single location (temperature) in the middle of the trace. As in the RBF case, significant noise power is allocated to the sensitive location itself, $X_i$, to limit \emph{direct} privacy loss. However, the noise added to the remainder of the trace is significantly different. It is tailored to thwart inference by a periodic prior, wherein the location one period away has correlation 1. 
			\vspace{2mm}\\
			The second noise mechanism (Mech. comp.) shows the covariance chosen by $\text{SDP}_\text{A}$ to protect a compound secret of two locations, $X_i, X_j$, 16 timesteps apart (not quite a full period). Here, we see the SDP randomize the phase of the additive noise such that periodic inference cannot tell directional information like $X_i > X_j$ or vice versa. 
			\vspace{2mm}\\
			The third noise mechanism (Mech. all basic) is identical to the all basic secrets mechanism chosen for the RBF case above, except using a periodic prior $\Sigma$. The mechanism chosen looks similar to the uniform baseline, except with slightly periodic off-diagonal correlation imitating the prior covariance. Additionally, noise power is mitigated towards the middle and ends of the trace. Again, \textbf{Figure \ref{fig: PER all}} indicates that this subtle change makes a significant difference in thwarting Bayesian adversaries. 
			}
		\label{fig: cov table rbf}
	\end{subfigure}
\end{figure*}

\clearpage

\subsection{Proof of results}
\label{apx: proofs} 
\subsubsection{Proof of Theorem \ref{thm: prior misspecification}} 
\textbf{Theorem \ref{thm: prior misspecification}} Prior-Posterior Gap:
\textit{
An $(\varepsilon, \lambda)$-CIP mechanism with conditional prior class $\Theta$ guarantees that for any event $O$ on sanitized trace $Z$
	\begin{align*}
		\bigg| \log \frac{P_{\calP, \calA}(s_i | Z \in O)}{P_{\calP, \calA}(s_j | Z \in O)} - \log \frac{P_{\calP}(s_i)}{P_{\calP}(s_j)} \bigg| \leq \varepsilon'
	\end{align*}
	for any $\calP \in \Theta$ with probability $\geq 1 - \delta$ over draws of $Z|\Xs=s_i$ or $Z|\Xs=s_j$, where $\varepsilon'$ and $\delta$ are related by
	\begin{align*}
		\varepsilon' = \varepsilon + \frac{\log \nicefrac{1}{\delta}}{\lambda - 1} \ .
	\end{align*}
	This holds under the condition that $Z|\Xs = s_i$ and $Z|\Xs = s_j$ have identical support. 
}

\begin{proof}
	This result makes use of a R\'enyi divergence property identified in \cite{renyi}: 
	\begin{lemma}
		\label{lem: renyi to eps delt}
		Let $\calP,\calQ$ be two distributions on $X$ of identical support such that  
		\begin{align*}
			\max \bigg\{ D_\lambda \binom{P_\calP(X)}{P_\calQ(X)}, 
			D_\lambda \binom{P_\calQ(X)}{P_\calP(X)} \bigg\}
			\leq \varepsilon 
		\end{align*}
		Then for any event $O$,
		\begin{align*}
			P_\calP(X \in O) \leq \max \big\{ e^{\varepsilon'} P_\calQ(X \in S), \delta \big\}
		\end{align*} 
		and
		\begin{align*}
			P_\calQ(X \in O) \leq \max \big\{ e^{\varepsilon'} P_\calP(X \in S), \delta \big\}
		\end{align*} 
		where 
		\begin{align*}
			\varepsilon' = \varepsilon + \frac{\log \nicefrac{1}{\delta}}{\lambda - 1}
		\end{align*}
	\end{lemma}
	CIP guarantees that for all $\calP \in \Theta$ and all discriminative pairs $(s_i, s_j) \in \Spairs$ (which also includes $(s_j, s_i)$) 
	\begin{align*}
		D_\lambda \binom{P_{\calP, \calA}(Z | \Xs = s_i)}{P_{\calP,\calA}(Z | \Xs = s_j)} \leq \varepsilon
	\end{align*}
	and thus by Lemma \ref{lem: renyi to eps delt} we have for any event $O$ on $Z$
	\begin{align*}
		P_{\calP, \calA}(Z \in O | \Xs = s_i) 
		\leq \max \big\{ e^{\varepsilon'} P_{\calP, \calA}(Z \in O | \Xs = s_j), \delta \big\}
	\end{align*}
	and
	\begin{align*}
		P_{\calP, \calA}(Z \in O | \Xs = s_j) 
		\leq \max \big\{ e^{\varepsilon'} P_{\calP, \calA}(Z \in O | \Xs = s_i), \delta \big\}
	\end{align*}
	As such, given that $\Xs = s_i$ the probability of some event $\{Z \in W\}$ such that 
	\begin{align*}
		P_{\calP, \calA}(Z \in W | \Xs = s_i) 
		\geq  e^{\varepsilon'} P_{\calP, \calA}(Z \in W | \Xs = s_j)
	\end{align*}
	is no more than $\delta$. The same is true swapping $s_j$ for $s_i$. So, over draws of $Z | \Xs = s_i$ or $Z | \Xs = s_j$ we have that 
	\begin{align*}
		 \frac{P_{\calP, \calA}(Z \in O | \Xs = s_i)}{P_{\calP,\calA}(Z \in O | \Xs = s_j)} \leq e^{\varepsilon'}
		 \quad \text{and} \quad
		 \frac{P_{\calP, \calA}(Z \in O | \Xs = s_j)}{P_{\calP,\calA}(Z \in O | \Xs = s_i)} \leq e^{\varepsilon'}
	\end{align*}
	with probability $\geq 1 - \delta$, which is equivalent to the statement that 
	\begin{align*}
		-\varepsilon' 
		\leq \log  \frac{P_{\calP, \calA}(Z \in O | \Xs = s_i)}{P_{\calP,\calA}(Z \in O | \Xs = s_j)}
		&\leq \varepsilon' \\
		\bigg| \log \frac{P_{\calP, \calA}(s_i | Z \in O)}{P_{\calP, \calA}(s_j | Z \in O)} - \log \frac{P_{\calP}(s_i)}{P_{\calP}(s_j)} \bigg| 
		&\leq \varepsilon'
	\end{align*}
\end{proof}

\subsubsection{Proof of Lemma \ref{lem: renyi additive loss}}
\textbf{Lemma \ref{lem: renyi additive loss}} (CIP loss for additive mechanisms)
\textit{
	For an additive noise mechanism, a fully dependent trace as in \textbf{Figure \ref{fig:condensed model}}, and any prior $\calP$ on $X$ the CIP loss may be expressed as
	\begin{align}
		&D_\lambda \binom{P_{\calA, \calP}(Z | \Xs = s_i)}{P_{\calA, \calP}(Z | \Xs = s_j)}  
		&= \sum_{i \in \Is} \bigg[ D_\lambda \binom{P_\calA(Z_i | X_i = s_i)}{P_\calA(Z_i | X_i = s_j)} \bigg]
		+ D_\lambda \binom{P_{\calA, \calP}(\Zu | \Xs = s_i)}{P_{\calA, \calP}(\Zu | \Xs = s_j)} \notag
	\end{align}
}
\begin{proof}
\begin{align}
	D_\lambda \binom{P_{\calA, \calP}(Z | \Xs = x_s)}{P_{\calA, \calP}(Z | \Xs = x_s')} 
	&= D_\lambda \binom{P_{\calA}(\Zs | \Xs = x_s)P_{\calA, \calP}(\Zu | \Xs = x_s)}{P_{\calA}(\Zs | \Xs = x_s')P_{\calA, \calP}(\Zu | \Xs = x_s')} \tag{1} \\
	&=  D_\lambda \binom{P_{\calA}(\Zs | \Xs = x_s)}{P_{\calA}(\Zs | \Xs = x_s')} + D_\lambda \binom{P_{\calA, \calP}(\Zu | \Xs = x_s)}{P_{\calA, \calP}(\Zu | \Xs = x_s')} \tag{2} \\
	&= D_\lambda \binom{\prod_{i \in \Is } P_{\calA}(Z_i | X_i = x_i)}{\prod_{i \in \Is } P_{\calA}(Z_i | X_i = x_i')} + D_\lambda \binom{P_{\calA, \calP}(\Zu | \Xs = x_s)}{P_{\calA, \calP}(\Zu | \Xs = x_s')} \tag{3} \\
	&= \sum_{i \in \Is} \bigg[ D_\lambda \binom{P_\calA(Z_i | X_i = x_i)}{P_\calA(Z_i | X_i = x_i')} \bigg]
	+ D_\lambda \binom{P_{\calA, \calP}(\Zu | \Xs = x_s)}{P_{\calA, \calP}(\Zu | \Xs = x_s')} \tag{4}
\end{align}
Where line (1) uses the conditional independence seen in the graphical model of \textbf{Figure \ref{fig:graphical models}}. Line (2) is due to the fact that the two terms in line (1) are conditionally independent, allowing for separating into the sum of two separate divergences (which is an easily verifiable property of R\'enyi divergence evident from its definition in Equation \ref{eqn: renyi}). Line (3) is again from the conditional independence between the $Z_i$ for each $i \in \Is$ when conditioned on $\Xs$. Line (4) uses the same property of R\'enyi divergence used in Line (2): the terms in the product are conditionally independent allowing for the separation into the sum of multiple divergences.

\end{proof}

\subsubsection{Proof of Theorem \ref{thm: prior misspecification}}
\label{apx: prior misspecification proof}
\textbf{Thoerem \ref{thm: prior misspecification}}
Robustness to Prior Misspecification 
\textit{
	Mechanism $\calA$ satisfies $\varepsilon(\lambda)$-CIP for prior class $\Theta$. Suppose the finite mean true distribution $\calQ$ is not in $\Theta$. The CIP loss of $\calA$ against prior $\calQ$ is bounded by 
	\begin{align*}
		D_\lambda \binom{P_{\calA, \calQ}(Z | \Xs = s_i)}{P_{\calA, \calQ}(Z | \Xs = s_j)} \leq \varepsilon'(\lambda)
	\end{align*}
	where
	\begin{align*}
		\varepsilon'(\lambda) 
		&= \frac{\lambda - \frac{1}{2}}{\lambda - 1} \ \Delta(2\lambda) + 
		\Delta(4\lambda - 3) +
		\frac{2\lambda - \frac{3}{2}}{2\lambda - 2} \ \varepsilon(4 \lambda -2)
	\end{align*}
	and where $\Delta(\lambda)$ is
	\begin{align*}
		\inf_{\calP \in \Theta} \sup_{s_i \in \calS} \max \bigg\{ 
		D_\lambda \binom{P_{ \calP}(\Xu | \Xs = s_i)}{P_{ \calQ}(\Xu | \Xs = s_i)}, 
		D_\lambda \binom{P_{ \calQ}(\Xu | \Xs = s_i)}{P_{ \calP}(\Xu | \Xs = s_i)}
		\bigg\}
	\end{align*}
}
\begin{proof}
By `finite mean' distribution $\calQ$, we mean that all conditionals of $\calQ$ given some $\Xs$ have finite mean. Since a conditional prior class contains conditionals of one distribution with any offset (any mean value), this guarantees that $\Delta(\lambda)$ is achieved for some $\calP \in \Theta$. Intuitively, this prevents the pathological case of $\inf_{\calP \in \Theta}$ being a limit as the mean of $\calP \rightarrow \infty$, only asymptotically approaching $\Delta(\lambda)$. If the mean of $\calQ$ is finite, then the closest $\calP \in \Theta$ (in R\'enyi divergence) must also have finite mean, since any mean is attainable in a conditional prior class $\Theta$.

With this in mind, we make use of the following triangle inequality provided in \cite{renyi}: 
\begin{lemma}
	For distributions $\calP$, $\calQ$, $\calR$ on $X$ with common support we have
	\begin{align*}
		D_\lambda \binom{P_\calP(X)}{P_\calQ(X)} \leq 
		\frac{\lambda - \frac{1}{2}}{\lambda - 1} D_{2 \lambda} \binom{P_\calP(X)}{P_\calR(X)} 
		+ D_{2\lambda - 1} \binom{P_\calR(X)}{P_\calQ(X)}
	\end{align*}
\end{lemma}
In our case, we assume that the mechanism $\calA$ gives $Z|\Xs = x_s$ identical support for all $\Is, x_s$. Using this, we have 
\begin{align*}
	D_\lambda \binom{P_{\calA, \calQ}(\Zu | \Xs = x_s)}{P_{\calA, \calQ}(\Zu | \Xs = x_s')} 
	\leq \frac{\lambda - \frac{1}{2}}{\lambda - 1} D_{2\lambda} \binom{P_{\calA, \calQ}(\Zu | \Xs = x_s)}{P_{\calA, \calP}(\Zu | \Xs = x_s)}
	+  \textcolor{blue}{D_{2\lambda - 1} \binom{P_{\calA, \calP}(\Zu | \Xs = x_s)}{P_{\calA, \calQ}(\Zu | \Xs = x_s')}} \ \ . \\
\end{align*}
By a data processing inequality, the divergence of the first term is bounded by $\Delta(2\lambda)$ and the blue term may be bounded by a second application of the triangle inequality: 
\begin{align*}
	\textcolor{blue}{D_{2\lambda - 1} \binom{P_{\calA, \calP}(\Zu | \Xs = x_s)}{P_{\calA, \calQ}(\Zu | \Xs = x_s')}}
	&\leq \frac{2\lambda - \frac{3}{2}}{2\lambda - 2} D_{4\lambda - 2} \binom{P_{\calA, \calP}(\Zu | \Xs = x_s)}{P_{\calA, \calP}(\Zu | \Xs = x_s')}
	+ D_{4\lambda - 3} \binom{P_{\calA, \calP}(\Zu | \Xs = x_s')}{P_{\calA, \calQ}(\Zu | \Xs = x_s')}
\end{align*}
The first divergence is bounded by $\varepsilon(4\lambda - 2)$ and the second divergence is bounded by $\Delta(4\lambda - 3)$. Putting all this together we have the following upper bound 
\begin{align*}
	D_\lambda \binom{P_{\calA, \calQ}(\Zu | \Xs = x_s)}{P_{\calA, \calQ}(\Zu | \Xs = x_s')}
	\leq 
	\frac{\lambda - \frac{1}{2}}{\lambda - 1} \ \Delta(2\lambda) + 
		\Delta(4\lambda - 3) +
		\frac{2\lambda - \frac{3}{2}}{2\lambda - 2} \ \varepsilon(4 \lambda -2)
\end{align*}
\end{proof}

\subsubsection{Proof of Theorem \ref{thm:GP bound}}
\label{apx: GP bound proof}
\textbf{Theorem \ref{thm:GP bound}}
CIP loss bound for GP conditional priors:
\emph{
Let $\Theta$ be a GP conditional prior class. Let $\Sigma$ be the covariance matrix for $X$ produced by its kernel function. Let $\calS$ be the basic or compound secret associated with $\Is$, and $S$ be the number of unique times in $\Is$. The mechanism $\calA(X) = X + G = Z$, where $G \sim \calN(\mathbf{0}, \Sigmag)$, then satisfies $(\varepsilon, \lambda)$-Conditional Inferential Privacy $(\Spairs, r, \Theta)$, where 
\begin{align*}
	\varepsilon
	&\leq \frac{\lambda}{2} S r^2 \Big(  \frac{1 }{\sigma_s^2} + \alpha^*  \Big) 
\end{align*}
where $\sigma_s^2$ is the variance of each $G_i \in \Gs$ (diagonal entries of $\Sigmag_{ss}$) and $\alpha^*$ is the maximum eigenvalue of $\Sigmaeff = \big(\Sigma_{us} \Sigma_{ss}^{-1}\big)^\intercal \big( \Sigma_{u | s} + \Sigma_{uu}^{(g)} \big)^{-1} \big(\Sigma_{us} \Sigma_{ss}^{-1}\big)$. 
}

\begin{proof}
Again, the conditional prior class $\Theta$ is defined by a kernel function $i,j \rightarrow \text{Cov}(i,j)$, which -- given the indices of the trace $X$ -- induces a covariance matrix $\Sigma$ between all $X_i, X_j$. In practice, when the sampling rate of locations is non-uniform the kernel function may use the time-stamps of the points in the trace to assign high correlation to $X_i$ that are close in time and low correlation to $X_i$ that are far apart in time. Of course, correlation between $X_i$ that are different dimension (e.g. latitude and longitude) must be designed for the given application and may be completely independent. The kernel function can encode this as well. 

Recall from Equation \ref{eqn: renyi} that the R\'enyi divergence between two mean-shifted multivariate normal distributions, $\calP_1 = \calN(\mu_1, \Sigma)$ and $\calP_2 = \calN(\mu_2, \Sigma)$ is 
\begin{align*}
	D_\lambda \binom{\calP_1}{\calP_2} = \frac{\lambda}{2} (\mu_1 - \mu_2)^\intercal \Sigma^{-1} (\mu_1 - \mu_2)
\end{align*}
Now, for any prior $\calP \in \Theta$, we have that $X \sim \calN(\mu, \Sigma)$ for some $\mu$ and for $\Sigma$ defined by the kernel function. Again, $G \sim \calN(\mathbf{0}, \Sigmag)$. $\Is$ encodes the indices of a single location basic secret or a multi-location compound secret. Then, the divergence to bound for $(\varepsilon, \lambda)$-CIP$(\Spairs, r, \Theta)$ is 
\begin{align*}
	D_\lambda \binom{P_{\calA, \calP}(Z | \Xs = s_i)}{P_{\calA, \calP}(Z | \Xs = s_j)}
\end{align*}
for any 
\begin{align*} 
	(s_i, s_j) \in \Spairs = \{(x_s, x_s'):\|x_s - x_s'\|_2 \leq 2r\}
\end{align*}
if $\Is$ encodes a basic secret, or for any
\begin{align*}
	(s_i, s_j) \in \Spairs = \Big\{\big( \{x_{s1}, x_{s2}, \dots\}, \{x_{s1}', x_{s2}', \dots\}\big): \| x_{sk} - x_{sk}' \|_2 \leq 2r, \forall \ k\Big\} 
\end{align*} 
if $\Is$ encodes a compound secret. A discriminative pair $(s_i,s_j)$ is two real valued vectors $\in \R^{|\Is|}$, representing two hypotheses about the true values of $\Xs$. We denote the $m^\text{th}$ element as ${s_i}_m, {s_j}_m$. Let $f:\Is \rightarrow [|\Is|]$ be a mapping from each index $w \in \Is$ to its corresponding position in the vector $s_i$ or $s_j$ (where the value of $X_w$ is hypothesized). By Lemma \ref{lem: renyi additive loss}, the divergence can be written as  
\begin{align*}
	D_\lambda \binom{P_{\calA, \calP}(Z | \Xs = s_i)}{P_{\calA, \calP}(Z | \Xs = s_j)}
	&= \sum_{w \in \Is} \bigg[ D_\lambda \binom{P_\calA(Z_w | X_w = {s_i}_{f(w)})}{P_\calA(Z_w | X_w = {s_j}_{f(w)})} \bigg]
	+ D_\lambda \binom{P_{\calA, \calP}(\Zu | \Xs = x_s)}{P_{\calA, \calP}(\Zu | \Xs = x_s')} 
\end{align*}
where $P_\calA(Z_w | X_w = x) = \calN(x, \sigma_s^2)$ for all $w \in \Is$. Recall from the statement of the Theorem that we assume the diagonal entries of $\Sigma_{ss}$ all equal some value $\sigma_s^2$: we add the same noise variance to each point in the secret set, which is optimal under MSE constraints. Additionally, note that for the hypothesis $\Xs = x_s$, we know the distribution of $\Xu | \Xs = x_s \sim \calN(\mu_{u|s}, \Sigma_{u|s})$, where $\mu_{u|s} = \mu_u + \Sigma_{us} \Sigma_{ss}^{-1} (x_s - \mu_s)$ and $\Sigma_{u|s} = \Sigma_{uu} - \Sigma_{us}\Sigma_{ss}^{-1} \Sigma_{su}$. Notice that only $\mu_{u|s}$ depends on the actual value of $x_s$, and $\Sigma_{u|s}$ depends only on the indices of $\Is$. Being the sum of two normally distributed variables, we have that $(\Zu | \Xs = x_s) \overset{d}{=} (\Xu|\Xs = x_s) + \Gu = \calN(\mu_{u|s}, \Sigma_{u|s} + \Sigmag_{uu})$. Substituting this into the divergences above sum of divergences: 
\begin{align}
	&D_\lambda \binom{P_{\calA, \calP}(Z | \Xs = s_i)}{P_{\calA, \calP}(Z | \Xs = s_j)}
	= \sum_{m =1}^{|\Is|} \bigg[ D_\lambda \binom{\calN({s_i}_m, \sigma_s^2)}{\calN({s_j}_m, \sigma_s^2)} \bigg]
	+ D_\lambda \binom{\calN(\mu_{u|s_i}, \Sigma_{u|s} + \Sigmag_{uu})}{\calN(\mu_{u|s_j}, \Sigma_{u|s} + \Sigmag_{uu})} \tag{1} \\
	&=  \frac{\lambda}{2} \sum_{m = 1}^{|\Is|}  \frac{1}{\sigma_s^2} ({s_i}_m - {s_j}_m)^2 
	+  \frac{\lambda}{2} (\mu_{u|s_i} - \mu_{u|s_j})^\intercal (\Sigma_{u|s} + \Sigmag_{uu})^{-1} (\mu_{u|s_i} - \mu_{u|s_j})  \tag{2} \\
	&=  \frac{\lambda}{2 \sigma_s^2}   ({s_i} - {s_j})^\intercal({s_i} - {s_j}) 
	+  \frac{\lambda}{2} \big( \Sigma_{us} \Sigma_{ss}^{-1}(s_i - s_j) \big)^\intercal (\Sigma_{u|s} + \Sigmag_{uu})^{-1} \big( \Sigma_{us} \Sigma_{ss}^{-1}(s_i - s_j) \big)  \tag{3}  \\
	&= \frac{\lambda}{2 \sigma_s^2}   ({s_i} - {s_j})^\intercal({s_i} - {s_j}) 
	+  \frac{\lambda}{2} (s_i - s_j)^\intercal \Sigma_{ss}^{-1} \Sigma_{su}  (\Sigma_{u|s} + \Sigmag_{uu})^{-1} \Sigma_{us} \Sigma_{ss}^{-1} (s_i - s_j) \tag{4} 
\end{align}
Line (1) substitutes in the normal distributions given by our mechanism and conditional prior class. Line (2) substitutes in the closed-form expression for R\'enyi divergence between two mean-shifted normal distributions given in Equation \ref{eqn: renyi}. Line (3) substitutes in the expression for $\mu_{u|s}$ given above, and simplifies. To expand out this simplification in explicit steps: 
\begin{align*}
	(\mu_{u|s_i} - \mu_{u|s_j})
	&= \big(  \mu_u + \Sigma_{us} \Sigma_{ss}^{-1} (s_i - \mu_s) -  [\mu_u + \Sigma_{us} \Sigma_{ss}^{-1} (s_j - \mu_s)] \big) \\
	&= \big(  \Sigma_{us} \Sigma_{ss}^{-1} s_i -  \Sigma_{us} \Sigma_{ss}^{-1} s_j \big) \\
	&= \Sigma_{us} \Sigma_{ss}^{-1} (s_i - s_j)
\end{align*}
Line (4) distributes the transpose in the right term of line (3): 
\begin{align*}
	\big( \Sigma_{us} \Sigma_{ss}^{-1}(s_i - s_j) \big)^\intercal
	&= (s_i - s_j)^\intercal \big(  \Sigma_{us} \Sigma_{ss}^{-1} \big)^\intercal \\
	&=  (s_i - s_j)^\intercal  \big( \Sigma_{ss}^{-1} \big)^\intercal \Sigma_{us}^\intercal   \\
	&= (s_i - s_j)^\intercal \Sigma_{ss}^{-1}  \Sigma_{su}
\end{align*}
where that final step is a consequence of $\Sigma$ being symmetric. $\Sigma_{ss}$ is also a symmetric matrix (so its inverse is symmetric) and $\Sigma_{us}^\intercal = \Sigma_{su}$. 

Returning to line (4) above, simplify this expression by substituting $\Delta = s_i - s_j$: 
\begin{align}
	D_\lambda \binom{P_{\calA, \calP}(Z | \Xs = s_i)}{P_{\calA, \calP}(Z | \Xs = s_j)}
	&= \frac{\lambda}{2 \sigma_s^2}   \Delta^\intercal \Delta 
	+  \frac{\lambda}{2} \Delta^\intercal \Sigma_{ss}^{-1} \Sigma_{su}  (\Sigma_{u|s} + \Sigmag_{uu})^{-1} \Sigma_{us} \Sigma_{ss}^{-1} \Delta \tag{5} \\
	&= \frac{\lambda}{2 \sigma_s^2}  \| \Delta \|_2^2 
	+  \frac{\lambda}{2} \Delta^\intercal \Sigmaeff \Delta \tag{6} 
\end{align}
Where $\Sigmaeff = \Sigma_{ss}^{-1} \Sigma_{su}  (\Sigma_{u|s} + \Sigmag_{uu})^{-1} \Sigma_{us} \Sigma_{ss}^{-1}$. The left term of line (6) attributes the direct loss of $\Zs$ on $\Xs$ and the right term attributes the indirect loss of $\Zu$ on $\Xs$. 

We are interested in bounding the expression of line (6) for all $(s_i, s_j) \in \Spairs$. We do this by bounding it for all vectors $\Delta \in \calD$ 
\begin{align*}
	\calD = \{ s_i - s_j : \| s_i - s_j \|_2 \leq  \sqrt{S}\  r \}
\end{align*}  
, where $S$ is the number of basic secrets (locations) contained in $\Is$ which may be a basic or compound secret set. For a basic secret ($S = 1$), this bound is tight, since $\calD = \{s_i - s_j: (s_i, s_j) \in \Spairs\}$. The set of $\Delta \in \calD$ is exactly any two hypothesis $(s_i, s_j)$ that are within any circle of radius $r$. For a compound secret, this bound is not guaranteed to be tight. Recall once again that the set of $\Spairs$ for a compound secret is given by the set of $(s_i, s_j)$ in 
\begin{align*}
	\Spairs = \Big\{\big( \{x_{s1}, x_{s2}, \dots\}, \{x_{s1}', x_{s2}', \dots\}\big): \| x_{sk} - x_{sk}' \|_2 \leq r, \forall \ k\Big\} 
\end{align*} 
For concreteness, consider the 2d location trace example in \textbf{Figure \ref{fig:nyc_example}}, where we have a compound secret of $S = 3$ locations. Here, $s_i, s_j \in \R^{6}$, where 6 comes from the fact that we have three 2d locations. So, $(s_i, s_j)$ represents a pair of hypotheses on all three locations. $s_i$'s hypothesis of the first secret location --- written as ${x_s}_1 \in \R^2$ above --- is within $r$ of the $s_j$'s hypothesis of the first secret location --- written as ${x_s}_1' \in \R^2$ above. The same goes for the second and third locations. So, the $L_2$ norm of $\Delta = s_i - s_j$ is no greater than
\begin{align*}
	\sup_{(s_i, s_j) \in \Spairs} \|s_i - s_j\|_2 
	&=  \sup_{(s_i, s_j) \in \Spairs} \sqrt{\sum_{m=1}^6 ({s_i}_m - {s_j}_m)^2} \\
	&=  \sup_{(s_i, s_j) \in \Spairs} \sqrt{\sum_{k=1}^3 \|{x_s}_k - {x_s}_k'\|_2^2} \\
	&= \sqrt{\sum_{k=1}^3 r^2} \\
	&= \sqrt{3} \ r
\end{align*}
For compound secrets, $\calD$ represents the $L_2$ ball enclosing all $\Delta \in \{s_i - s_j : (s_i, s_j) \in \Spairs \}$. However, $\calD$ also includes some values of $\Delta = s_i - s_j$ not covered by $\Spairs$. Suppose an adversary considers the hypotheses 
\begin{align*}
s_i = \{x_{s1}, x_{s2}, x_{s3}\}, s_j = \{x_{s1}', x_{s2}', x_{s3}'\}
\end{align*} 
where ${x_s}_1 = 0, {x_s}_1' = \sqrt{3} \ r, {x_s}_2 = {x_s}_2', {x_s}_3 = {x_s}_3'$. Since ${x_s}_1, {x_s}_1'$ are not within $r$ of each other, this is not in $\Spairs$. However, it is covered by $\calD$, and thus is covered by our bound on CIP loss and our mechanisms. 

With $\calD$ defined, we may return to bounding the expression in line (6): 
\begin{align}
	D_\lambda \binom{P_{\calA, \calP}(Z | \Xs = s_i)}{P_{\calA, \calP}(Z | \Xs = s_j)}
	&\leq \sup_{\Delta \in \calD} \bigg( \frac{\lambda}{2 \sigma_s^2}  \| \Delta \|_2^2 
	+  \frac{\lambda}{2} \Delta^\intercal \Sigmaeff \Delta \bigg) \tag{7} \\
	&\leq  \frac{\lambda}{2}\bigg( \frac{1}{\sigma_s^2} S r^2 + S r^2 \text{maxeig}(\Sigmaeff) \bigg) \tag{8} \\
	&= \frac{\lambda}{2} S r^2 \big( \frac{1}{\sigma_s^2} + \alpha^* \big) \tag{9}
\end{align}
where line (8) distributes the supremum. For the right term, this is given by the maximum magnitude of all $\Delta \in \calD$ times the maximum eigenvalueof $\Sigmaeff$ which equals $S r^2 \text{maxeig}(\Sigmaeff)$. Line (9) simply substitutes $\alpha^* = \text{maxeig}(\Sigmaeff)$. 

\end{proof}

\subsubsection{Proof of Corollary \ref{cor: composition}}
\textbf{Corollary \ref{cor: composition}}
Graceful Composition in Time
\textit{
	Suppose a user releases two traces $X$ and $\hat{X}$ with additive noise $G \sim \calN(\mathbf{0}, \Sigmag)$ and $\hat{G} \sim \calN(\mathbf{0}, \hat{\Sigma}^{(g)})$, respectively. Then basic or compound secret $\Xs$ of $X$ enjoys $(\bar{\varepsilon}, \lambda)$-CIP, where 
	\begin{align*}
		\bar{\varepsilon} \leq \frac{\lambda}{2} S r^2 \Big(  \frac{1 }{\sigma_s^2} + \bar{\alpha}^*  \Big) 
	\end{align*}
	and where $\bar{\alpha}$ is the maximum eigenvalue of $\bar{\Sigma}_{\text{eff}} = \big(\Sigma_{us} \Sigma_{ss}^{-1}\big)^\intercal \big( \Sigma_{u | s} + \bar{\Sigma}_{uu}^{(g)} \big)^{-1} \big(\Sigma_{us} \Sigma_{ss}^{-1}\big)$. $\Sigma$ is the covariance matrix of the joint distribution on $X, \hat{X}$ and 
	\begin{align*}
	\bar{\Sigma}^{(g)} =
		\begin{bmatrix}
			 \Sigmag & 0 \\
			 0 &  \hat{\Sigma}^{(g)} \ .
		\end{bmatrix}
	\end{align*}
}

\begin{proof}
Here, we record two traces (presumably) far apart in time 
\begin{align*}
	(X_1, \dots, X_n) \text{ and } (\hat{X}_1, \dots, \hat{X}_m)
\end{align*}
And release
\begin{align*}
	(Z_1, \dots, Z_n) = (X_1, + G_1, \dots, X_n + G_n) \text{ and } (\hat{Z}_1, \dots, \hat{Z}_m) = (\hat{X}_1, + \hat{G}_1, \dots, \hat{X}_m, + \hat{G}_m)
\end{align*}
the first trace protects secret locations $\Xs$ and the second protects $\widehat{\Xs}$, so we have that 
\begin{align*}
	D_\lambda \binom{P_{\calA, \calP}(Z | \Xs = s_i)}{P_{\calA, \calP}(Z | \Xs = s_j)} &\leq \varepsilon \\
	D_\lambda \binom{P_{\calA, \calP}(\hat{Z} | \widehat{\Xs} = \hat{s}_i)}{P_{\calA, \calP}(\hat{Z} | \widehat{\Xs} = \hat{s}_j)} &\leq \hat{\varepsilon}
\end{align*}
We aim to update the losses: 
\begin{align*}
	D_\lambda \binom{P_{\calA, \calP}(Z, \hat{Z} | \Xs = s_i)}{P_{\calA, \calP}(Z, \hat{Z} | \Xs = s_j)} &\leq \varepsilon' \\
	D_\lambda \binom{P_{\calA, \calP}(\hat{Z}, Z | \widehat{\Xs} = \hat{s}_i)}{P_{\calA, \calP}(\hat{Z}, Z | \widehat{\Xs} = \hat{s}_j)} &\leq \hat{\varepsilon}'
\end{align*}
Fortunately, our framework is pretty friendly to figuring this out, and can be done simply by updating the `inferential loss term' $\alpha^*$ and $\hat{\alpha}^*$ of each, the max eigenvalues used to compute each of $\varepsilon$ and $\hat{\varepsilon}$, respectively. Let's focus on $\varepsilon'$, since the same analysis follows for $\hat{\varepsilon}'$.  

Recall that $\alpha^*$ is given by the max eigenvalue of $\Sigmaeff$ which is 
\begin{align*}
	\Sigmaeff 
	&= \big(\Sigma_{us} \Sigma_{ss}^{-1}\big)^\intercal \big( \Sigma_{u | s} + \Sigma_{uu}^{(g)} \big)^{-1} \big(\Sigma_{us} \Sigma_{ss}^{-1}\big)
\end{align*}
Where $\Sigma$ is the covariance matrix of $X_1, \dots, X_n$ and $\Sigmag$ is the noise covariance matrix added. Simply augment $\Sigma$ to become the joint covariance matrix $\Sigma_J$ of $X, \hat{X}$, and augment $\Sigmag$ to become 
\begin{align*}
	\Sigmag_J
	&= 
	\begin{bmatrix}
		\Sigmag & 0 \\
		0 & \hat{\Sigma}^{(g)}
	\end{bmatrix}
\end{align*}
then update $\Sigmaeff$ to $\Sigma_{\text{eff}, J}$ which uses both $\Sigma_J$ and $\Sigmag_J$. Using the corresponding max eigenvalue $\alpha^*_J$ in the loss expression of Theorem 3.2 gives us $\varepsilon'$. 

Note that for kernels like RBF, $\varepsilon' \rightarrow \varepsilon$ as the traces $X$ and $\hat{X}$ move apart further and further in time. This is not the case for traces using a purely periodic kernel with not time decay, and we should expect much worse composition. 
\end{proof}

\subsubsection{Traces with Independent Dimensions}
In many cases, the different dimensions of the trace may be probabilistically independent, and it may be more convenient to make separate privacy mechanisms for each. For a 2d trace $X$, suppose $\Ix$ and $\Iy$ store the indices of the latitude points $\Xx$ and longitude points $\Xy$, such that $X = \Xx \cup \Xy$. If latitude and longitude are independent, it may be more convenient to characterize the conditional priors of $\Xx$ abd $\Xy$ separately. The question is whether privacy guarantees remain for the full trace $X$. To answer this, we provide the following corollary: 

\begin{corollary}\emph{CIP loss of independent dimensions} 
\label{cor: independence}
	Let $\Theta$ be a GP conditional prior class on a 2d trace $X$ such that the dimensions are independent. Let $\Is$ be some secret set of time indices corresponding to some basic or compound secret. For the trace $X = \Xx \cup \Xy$, the Gaussian mechanism $\calA(X) = \Zx \cup \Zy$ where $\Zx = \calA_x(\Xx) = \Xx + \Gx$ and $\Zy = \calA_y(\Xy) = \Xy + \Gy$ satisfies $(\varepsilon, \lambda)$-CIP where
	\begin{align*}
		\varepsilon \leq \frac{\lambda}{2} S r^2 \big( \frac{1}{\sigma_s^2} + \alpha^*_x + \alpha^*_y \big) 
	\end{align*} 
	when $\calA_x$ and $\calA_y$ provide $\frac{\lambda}{2} S r^2 \big( \frac{1}{\sigma_s^2} + \alpha^*_x)$ and $\frac{\lambda}{2} S r^2 \big( \frac{1}{\sigma_s^2} + \alpha^*_y)$ to $\Is \cap \Ix$ and $\Is \cap \Iy$, respectively. 
\end{corollary}
The gist of this corollary is that a mechanism can be designed to achieve the bound of Theorem \ref{thm:GP bound} to each dimension independently and released with still-meaningful privacy guarantees. The reason is that this still includes all secret pairs $\Spairs$ 
\begin{proof}
	By independence, $\Xx$ and $\Xy$ can be treated as two unconnected traces of the type seen in \textbf{Figure \ref{fig:graphical models}}. As such the privacy guarantee of Theorem \ref{thm:GP bound} can be upheld for each. The question is whether bounding CIP loss to the one-dimensional basic or compound secret associated with secret sets $\Is \cap \Ix$ and $\Is \cap \Iy$ still provides guarantees for the full secret set $\Is$. 
	
	Without loss of generality, we will demonstrate for a basic and a compound secret. Consider the basic secret set $\Is = \{X_{10}, X_{11}\}$, where $\Is \cap \Ix = \{X_{10}\}$ (latitude) and $\Is \cap \Iy = \{X_{11}\}$ (longitude). We again assume that independent gaussian noise of variance $\sigma_s^2$ is added to all $\Xs$, since this is optimal under utility constraints. We have now bounded the R\'enyi divergence when conditioning on pairs of hypotheses on latitude and longitude separately. 
	\begin{align*} 
	{\Spairs}_x = {\Spairs}_y = \{(x_s, x_s'):x_s \in \R,  \|x_s - x_s'\|_2 \leq r\}
	\end{align*}
	By independence, this also bounds the R\'enyi divergence conditioning on pairs of hypotheses on latitude and longitude jointly: 
	\begin{align*} 
	{\Spairs}_{xy} = \{(x_s, x_s'):x_s \in \R^2,  \|x_s - x_s'\|_2 \leq r\}
	\end{align*}
	In effect, we have guaranteed privacy for any pair of hypotheses $(s_i, s_j)$ in the square circumscribing the circle of radius $r$ that we with to provide. The analysis on the direct privacy loss is exactly the same as it was in the more general case. Since the R\'enyi divergences of $\Xu \cap \Xx$ and of $\Xu \cap \Xy$ add, the $\alpha^*$'s add. 
	
	The same goes for a compound secret. Consider three location compound secret pairs given by 
	\begin{align*}
		{\Spairs}_{xy} = \Big\{\big( \{x_{s1}, x_{s2}, \dots\}, \{x_{s1}', x_{s2}', \dots\}\big): x_{si} \in \R^2, \| x_{sk} - x_{sk}' \|_2 \leq r, \forall \ k\Big\} 
	\end{align*} 
	Instead, we bound privacy loss for 
	\begin{align*}
		{\Spairs}_x = {\Spairs}_y = \Big\{\big( \{x_{s1}, x_{s2}, \dots\}, \{x_{s1}', x_{s2}', \dots\} \big): x_{si} \in \R, \| x_{sk} - x_{sk}' \|_2 \leq r, \forall \ k \Big\}
	\end{align*}
	Separately, giving us $\alpha_x^*$ and $\alpha_y^*$. This again includes any two hypotheses on the three locations such that each pair of $x_{sk}, x_{sk}'$ is within a square circumscribing a circle of radius $r$. We achieve this by bounding privacy loss for all $\Delta_x$ in a 3d $L_2$ ball of radius $\sqrt{S}  \ r$, as with $\Delta_y$. 
	
	This corollary can be extended to all traces of all dimensions that are probabilistically independent. 
\end{proof}

We make use of the above proof in the Experiments section. 

\subsection{Derivation of Algorithms}
\label{apx: algorithmns}
In this section, we derive the three SDP-based algorithms of Section \ref{sec: algorithms} and their properties. 

\subsubsection{Derivation of $\text{SDP}_\text{A}$}

$\text{SDP}_\text{A}$ minimizes the privacy loss bound of Theorem \ref{thm:GP bound} for any compound or basic secret encoded by secret set $\Is$. As is clarified in its proof (Appendix \ref{apx: GP bound proof}), the bound is tight when $\Is$ encodes a basic secret. If $\Is$ encodes a compound secret, the tightness depends on the conditional prior class $\Theta$. 

Our variable for minimizing this bound is the noise covariance matrix $\Sigmag$. Due to the conditional independence exhibited by Lemma \ref{lem: renyi additive loss}, $\Gs$ and $\Gu$ may be independent. The additive noise $G_i \in \Gs$ are all independent Gaussian with variance $\sigma_s^2$. This is because --- conditioning on $\{\Xs = x_s\}$ --- $\Zs$ is independent of $\Xu$ and $\Zu$. So, $\Gs \sim \calN(\mathbf{0}, \sigma_s^2 I)$, and $\Sigmag_{ss} = \sigma_s^2 I$. The additive noise $G_i \in \Gu$ are all dependent as described by $\Sigmag_{uu}$, and $\Gu \sim \calN(\mathbf{0}, \Sigmag_{uu})$. Consequently, $\Sigmag$ is completely characterized by $\Sigmag_{uu}$ and $\sigma_s^2$. 

To see how the bound of Theorem \ref{thm:GP bound} can be redrafted as an SDP, first notice that its two terms may be written as the maximum eigenvalue of a matrix product. Here, $\Sigmaeff = A^\intercal B A$, where $A = \Sigma_{us} \Sigma_{ss}^{-1}$ and $B = \big( \Sigma_{u | s} + \Sigmag_{uu} \big)^{-1}$
\begin{align*}
	\frac{1}{\sigma_s^2} + \alpha^*
	= \text{maxeig} \big( 
	\frac{1}{\sigma_s^2} I + A^\intercal B A \big)
	= \text{maxeig} \bigg(  
	\begin{bmatrix}
		I \  A
	\end{bmatrix} 
	\begin{bmatrix}
		\frac{1}{\sigma_s^2} I \ \ \  0 \\
		\quad 0 \quad  B
	\end{bmatrix}
	\begin{bmatrix}
		I \\ A
	\end{bmatrix}
	\bigg) 
	= \text{maxeig} \big( \tilde{A}^\intercal \tilde{B} \tilde{A} \big) 
\end{align*}
This expression uses all parameters of $\Sigmag$: $\sigma_s^2$ parametrizes $\Sigmag_{ss}$ and $\Sigmag_{uu} = B^{-1} - \Sigma_{u|s}$, where $\Sigma_{u|s}$ is given by the kernel function of $\Theta$. 

Before casting this as an SDP, we provide a formal definition from \cite{SDPs}: 

\begin{definition}\emph{Semidefinite Program} 
	\label{def: SDP}
	The problem of minimizing a linear function of a variable $x \in \R^n$ subject to a matrix inequality: 
	\begin{align*}
		\min_{x \in \R^n} \ &c^\intercal x \\
		&\text{s.t. } F_0 + \sum_{i=1}^n x_i F_i \succeq 0 \\
		& \quad \ \  Ax = b
	\end{align*}
	where the $F_i \in \R^{n \times n}$ are all symmetric and $A \in \R^{p \times n}$ is a \emph{semidefinite program}, or SDP. 
\end{definition}

The task of minimizing $\text{maxeig} \big( \tilde{A}^\intercal \tilde{B} \tilde{A} \big)$ under MSE constraints can almost be formulated as an SDP: 
\begin{align*}
	\min_{B \succeq 0 , \nicefrac{1}{\sigma_s^2} \geq 0} \ &\beta^* \\
	&\text{s.t. } \beta^* I  \succeq \tilde{A}^\intercal \tilde{B} \tilde{A} \\
	& \quad \ \ B \preceq \Sigma_{u|s}^{-1} \\
	&\quad \ \ \trace(\Sigmag_{uu}) + |\Is| \sigma_s^2 \leq n o_t 
\end{align*}
Here, the first constraint guarantees that the maximum eigenvalue of $\tilde{A}^\intercal \tilde{B} \tilde{A}$ is bounded by $\beta^*$, which the objective minimizes. At program completion, we set $\Sigmag_{uu} = B^{-1} - \Sigma_{u|s}$, and the second constraints ensures that this is still PSD. The final constraint bounds the MSE of the mechanism $\Sigmag$. Note that $\trace(\Sigmag_{uu}) + |\Is| \sigma_s^2 = \trace(\Sigmag)$. The trouble lies the last constraint. Our program variable is $B$, but the final linear constraint requires $\Sigmag$, which is expressed using the inverse of $B$. This is not immediately available in the SDP framework. 

To make the final linear constraint available, we invert the above program using the observation that the maximum eigenvalue of $\tilde{A}^\intercal \tilde{B} \tilde{A}$ is the inverse of the minimum eigenvalue of $(\tilde{A}^\intercal \tilde{B} \tilde{A})^{-1}$. Instead of optimizing over $B$ and $\nicefrac{1}{\sigma_s^2}$, we optimize over $B^{-1}$ and $\sigma_s^2$. Since $B^{-1} = \Sigma_{u|s} + \Sigmag_{uu}$, we may now have a utility constraint directly on the trace of $\Sigmag$. To make $B^{-1}$ our program variable, we approximate $(\tilde{A}^\intercal \tilde{B} \tilde{A})^{-1}$ with $\tilde{A}^{-1} \tilde{B}^{-1} \tilde{A}^{-\intercal}$. First note that $\tilde{A} \in \R^{n \times |\Is|}$, and has full column rank for the covariances we work with. So, $\tilde{A}^{-1} = (\tilde{A}^\intercal \tilde{A})^{-1}\tilde{A}^\intercal \in \R^{(|\Is| \times n)}$ is the left inverse of $\tilde{A}$ and is the least squares solution to $\tilde{A}^{-1} \tilde{A} = \tilde{A}^\intercal \tilde{A}^{-\intercal}  = I$ (we denote its transpose as $\tilde{A}^{-\intercal}$). It is also the least squares solution to $\tilde{A} \tilde{A}^{-1} = \tilde{A}^{-\intercal} \tilde{A}^\intercal = I$. Thus, we have an approximation of the inverse $(\tilde{A}^\intercal \tilde{B} \tilde{A})^{-1}$: 
\begin{align*}
	(\tilde{A}^\intercal \tilde{B} \tilde{A}) \ (\tilde{A}^{-1} \tilde{B}^{-1} \tilde{A}^{-\intercal})
	&\approx \tilde{A}^\intercal \tilde{B} \tilde{B}^{-1} \tilde{A}^{-\intercal} \\
	&= \tilde{A}^\intercal \tilde{A}^{-\intercal} \\
	&\approx I
\end{align*}

We now can optimize in terms of $B^{-1}$ with the augmented matrix $\tilde{B}^{-1}$: 
\begin{align*}
	\tilde{B}^{-1} = 
	\begin{bmatrix}
		\sigma_s^2 I \ \ \  0 \\
		\quad 0 \quad  B^{-1}
	\end{bmatrix}
\end{align*}

We then optimize the following SDP: 

\begin{align*}
	\max_{B^{-1} \succeq 0 , \sigma_s^2 \geq 0} \ &\beta^* \\
	&\text{s.t. } \beta^* I  \preceq \tilde{A}^{-1} \tilde{B}^{-1} \tilde{A}^{-\intercal} \\
	& \quad \ \ B^{-1} \succeq \Sigma_{u|s} \\
	&\quad \ \ \trace(\tilde{B}) -  \trace{(\Sigma_{u|s})} \leq n o_t 
\end{align*}
Upon program completion we recover $\sigma_s^2$ and $\Sigmag_{uu} = B^{-1} - \Sigma_{u|s}$ which we know is PSD due to the second constraint. The first constraint guarantees that the minimum eigenvalue of the approximated inverse is $\geq \beta^*$, which the objective maximizes. If the minimum eigenvalue of the approximate inverse is close to that of the true inverse, then we successfully minimize the maximum eigenvalue of $\tilde{A}^\intercal \tilde{B} \tilde{A}$, and thus minimize the direct and indirect privacy loss. The third constraint limits the MSE of $\Sigmag$ since $\trace(\tilde{B}) - \trace(\Sigma_{u|s}) = (\trace(\Sigmag_{uu}) + |\Is| \sigma_s^2 + \trace(\Sigma_{u|s})) - \trace(\Sigma_{u|s}) = \trace(\Sigmag)$. By inverting $\tilde{A}^\intercal \tilde{B} \tilde{A}$, this constraint is available in the SDP framework. 

By expressing the above program in terms of the variable $\Sigmag$ instead of indirectly via $B^{-1}$ and $\sigma_s^2$, we get $\text{SDP}_\text{A}$: 

\begin{align*}
	\textbf{SDP}_\textbf{A}: \quad 
	\argmax_{\Sigmag \succeq 0}& \ \beta^* \\
	\text{s.t. }& \tilde{A}^{-1} \tilde{B}^{-1} \tilde{A}^{-\intercal} \succeq \beta^* \mathbf{I} \\
	&\trace(\Sigmag) \leq n o_t
\end{align*}
It is straightforward to write this SDP in the form seem in Definition \ref{def: SDP}. The program variables $x$ would be the diagonal and upper or lower triangular part of $\Sigmag$ along with $\beta^*$. With some linear algebra, the first constraint can be written in the form of $F_0 + \sum_{i=1}^n x_i F_i \succeq 0$, and the second constraint can be written as $Ax = b$. With the use of contemporary convex programming tools like CVXOPT \citep{cvxopt} rewriting into this form is unnecessary.

\subsubsection{Derivation of $\text{SDP}_\text{B}$ }
\label{apx: SDP B}
$\text{SDP}_\text{B}$ takes a set of covariance matrices $\calF = \{\Sigma_1, \dots, \Sigma_k\}$, each of which is designed to protect some secret set ${\Is}_i$, and returns a covariance matrix $\Sigmag$ that preserves the privacy loss bound of each $\Sigma_i$ to each ${\Is}_i$. It does so while minimizing the utility loss of $\Sigmag$. This algorithm is also expressed as an SDP. It is based on the following corollary, which we have omitted from the main text: 
\begin{corollary}\emph{More PSD, More Private: }
\label{cor: more_psd}
	For a basic or compound secret denoted by indices $\Is$, the CIP loss bound of Equation \ref{eqn: priv bound} provided by a Gaussian noise mechanism with covariance $\Sigmag$ is lower than it would be for any ${\Sigmag}' \prec \Sigmag$. 
\end{corollary}
\begin{proof}
	First note that if $\Sigmag \succ {\Sigmag}' $, then the same is true for its sub-matrices: 
	\begin{align*}
		\Sigmag_{ss} \succ {\Sigmag_{ss}}'
		\quad \quad
		\Sigmag_{uu} \succ {\Sigmag_{uu}}'
	\end{align*}
	Recall the privacy loss bound of Equation \ref{eqn: priv bound}: 
	\begin{align*}
		\varepsilon \leq \frac{\lambda}{2} S r^2 \Big(  \frac{1 }{\sigma_s^2} + \alpha^*  \Big)
	\end{align*}
	Also recall that $\Sigmag_{ss} = \sigma_s^2 I$ and ${\Sigmag_{ss}}' = {\sigma_s^2}' I$. Since $\Sigmag_{ss} \succ {\Sigmag_{ss}}'$, we already know that $\sigma_s^2 > {\sigma_s^2}'$, and thus the first term of Equation \ref{eqn: priv bound} is lower for $\Sigmag$.
	
	It remains to show that the second term is also lower, $\alpha^* < {\alpha^*}'$. Starting with what we're given, 
	\begin{align*}
		\Sigmag_{uu} &\succ {\Sigmag_{uu}}' \\
		\Sigmag_{uu} + \Sigma_{u|s} &\succ {\Sigmag_{uu}}' + \Sigma_{u|s} \\
		(\Sigmag_{uu} + \Sigma_{u|s})^{-1} &\prec ({\Sigmag_{uu}}' + \Sigma_{u|s})^{-1} \\
		B &\prec B' \\
		A^\intercal B A &\prec A^\intercal B' A \\
		\maxeig(A^\intercal B A) &< \maxeig(A^\intercal B' A) \\
		\alpha^* &< {\alpha^*}'
	\end{align*}
	Therefore $\frac{1}{\sigma_s^2} + \alpha^* < \frac{1}{{\sigma_s^2}'} + {\alpha^*}'$, and the CIP bound of Equation \ref{eqn: priv bound} is lower for $\Sigmag$ than it is for ${\Sigmag}'$. 
\end{proof}
With Corollary \ref{cor: more_psd} in mind, $\text{SDP}_\text{B}$ is natural: 

\begin{align*}
	\textbf{SDP}_\textbf{B}: \quad 
	\argmin_{\Sigmag } \  &\trace(\Sigmag) \\
	\text{s.t. }& \Sigmag \succeq \Sigmag_i , \ \forall \Sigmag_i \in \calF
\end{align*}

$\text{SDP}_\text{B}$ attempts to minimize, but does not constrain, the utility loss of the chosen $\Sigmag$. To provide an upper bound on the resulting utility loss, we provided the following claim in the main text: 

\textbf{Claim} Utility loss of $\text{SDP}_\text{B}$: 
\emph{
	The utility loss of $\Sigmag = \text{SDP}_\text{B}(\calF)$ is no greater than $\sum_{\Sigma_i \in \calF} \trace(\Sigma_i)$. 
}
\begin{proof}
	The covariance ${\Sigmag}' = \sum_{\Sigmag_i \in \calF} \Sigmag_i$ with MSE $\sum_{\Sigmag_i \in \calF} \trace(\Sigmag_i)$ is in the feasible set of $\text{SDP}_\text{B}$ problem since ${\Sigmag}' \succeq \Sigmag_i, \ \forall \Sigmag_i \in \calF$. Unless ${\Sigmag}'$ has the lowest MSE of all $\Sigmag$ in the feasible set, a covariance matrix with better utility will be chosen. 
\end{proof}

\subsubsection{Derivation of Algorithm \ref{alg: Multiple Secrets}, Multiple Secrets}

Multiple Secrets combines $\text{SDP}_\text{A}$ and $\text{SDP}_\text{B}$ to minimize the privacy loss to each basic secret within a trace. The basic mechanism is useful in cases when inferences at each time within the trace --- each basic secret --- is sensitive. 

Let ${\Is}_i$ be the secret set representing basic secret $i$, of which there are $N$ (e.g. if location is sampled at $N$ times). Then $\mathbb{I}_{\calS_b} = \{{\Is}_1, \dots, {\Is}_N\}$ contains the indices corresponding to each. Multiple Secrets works by first producing $N$ covariance matrices, $\Sigmag_i$ = $\text{SDP}_\text{A}$$({\Is}_i, \Sigma, o_t)$ on each basic secret. It then uses $\text{SDP}_\text{B}$($\calF = \{\Sigmag_1, \dots, \Sigmag_N\}$) to produce a single covariance matrix $\Sigmag$ that preserves the privacy loss to each basic secret (note that, being basic secrets, the privacy loss bound that SIG OPT optimizes is tight). 

By virtue of using $\text{SDP}_\text{B}$, the MSE of the resultant $\Sigmag$ is minimized but not constrained. To bound the MSE of the Basic Mechanism by $O$, we may simply bound the MSE of each $\Sigmag_i$ by $o_t = \nicefrac{O}{N}$. Then, by the above Claim, the MSE of the solution cannot be greater than $O$. In practice, this bound may be too loose. We hope to tighten it in future work. 

\subsection{Experimental details}
\label{apx: experiments}

We use a 2d location trace and a 1d home temperature dataset. For the location data, having observed that the correlation between latitude and longitude is low ($ \approx 0.06$) we treat each dimension as independent. By way of Corollary \ref{cor: independence}, this allows us to bound privacy loss and design mechanisms for each dimension separately. Furthermore, having observed that each dimension fits the nearly the same conditional prior, we treat our dataset of 10k 2-dimensional traces as a dataset of 20k 1-dimensional traces, where each trace represents one dimension of a 2d location trajectory. 

The one-dimensional traces of temperature and location are indexed by timestamps, for which we would use the following kernel functions: 

\begin{align}
	k_{\text{RBF}}(t_i, t_j) 
	=  \sigma_x^2 \exp \Big( -\frac{(t_i - t_j)^2}{2 l^2} \Big) 
	\quad \quad 
	k_{\text{PER}}(t_i, t_j) 
	=  \sigma_x^2 \exp \Big(  \frac{-2 \sin^2(\pi |t_i - t_j| / p)}{l^2} \Big)
\end{align}

to determine the covariance between two points sampled at times $t_i$ and $t_j$. The parameters including variance $\sigma_x^2$ and length scale $l$. The lengthscale determines the window of time in which two sampled points are highly correlated. 

\paragraph{Preprocessing of location data} We first limit the dataset to traces of under 50 locations that are between 4.5 and 5.5 minutes in duration. Caring only about the conditional dependence between locations, we then de-mean each trace and normalize its variance to one. Normalizing the variance of traces implicitly sets $\sigma_x^2 = 1$ in the above RBF kernel, in essence assuming that the adversary has a decent prior for the user's average speed in a given trace, and could do the same operation. 

\paragraph{Fitting of location data} We then find the maximum likelihood RBF kernel for each distinct trace. Having fixed the variance $\sigma_x^2$, this amounts to fitting only the length scale for each dimension, $l_x$ and $l_y$, individually. The length scale represents the average window of time during which neighboring locations are highly correlated (i.e. correlation $ > 0.8$). Relatively smooth traces will have large length scales and chaotic traces will have low length scales. However, the fact that sampling rates vary significantly between traces means that traces with equal length scales can have very different degrees of correlation. To encapsulate both of these effects, we study the empirical distribution of \emph{effective} length scale of each trace
\begin{align*}
	l_{\text{eff},x} = \frac{l_x}{P}
	\quad
	l_{\text{eff},y} = \frac{l_y}{P}
\end{align*}
where $P$ is the trace's sampling period and $l_x,l_y$ are the its optimal length scales. $l_{\text{eff},x},l_{\text{eff},y}$ tell us the average number of neighboring locations that are highly correlated, instead of time period. For instance, a given trace with an optimal $l_{\text{eff},x} = 8$ tells us that every eight neighboring location samples in the $x$ dimension have correlation $> 0.8$. The empirical distribution of effective length scales across all traces describes -- over a range of logging devices (sampling rates), users, and movement patterns -- how many neighboring points are highly correlated in location trace data. After this preprocessing, we are able to use the kernels that take indices (not time) as arguments. 

\begin{align*}
	\label{eqn: kernels}
	k_{\text{RBF}}(i, j) 
	=  \exp \Big( -\frac{(i - j)^2}{2\leff^2} \Big) 
	\quad \quad 
	k_{\text{PER}}(i, j) 
	=  \exp \Big(  \frac{-2 \sin^2(\pi |i - j| / p)}{\leff^2} \Big)
\end{align*}

In each plot we then observed a spectrum of conditional priors by sweeping the effective length scale and plotting posterior uncertainty for various noise mechanisms of equal utility loss. This ranges from a prior assuming nearly independent location samples (chaotic trace) on the left up to highly dependent location samples (traveling in a straight line or standing still) on the right. To understand how realistic these conditional prior parameters are, we displayed the middle 50\% of the empirical distribution of $l_{\text{eff}}$ ($x$ and $y$ together) from the GeoLife dataset. Note that the distribution of ${\leff}_x$ and ${\leff}_y$ are nearly identical. 

To compute posterior uncertainty, we consider a 50-point one-dimensional location trace. The basic secret is a single index in the middle of the trace, and the compound secret consists of two neighboring indices also in the middle of trace. For each value of $l_{\text{eff}}$, we compute the $\R^{50 \times 50}$ conditional prior covariance matrix $\Sigma$ using the RBF kernel above. We then compare the posterior uncertainty when $\Sigmag$ is an Approach C baseline, or an optimized covariance matrix using one of the three algorithms. We re-optimize $\Sigmag$ for each $\leff$, since each $\leff$ represents a different conditional prior class. The MSE is fixed in all figures except the two exhibiting ``All Basic Secrets'', where $\text{SDP}_\text{B}$ is used. Recall that this algorithm minimizes utility loss while maintaining a series of privacy guarantees. Here, the MSE is identical across mechanisms for each $\leff$, but changes from one $\leff$ to another. 

For the temperature data, our preprocessing steps were nearly identical, except we use the periodic kernel instead of the RBF kernel, and we did not need to remove any traces from the dataset, as the data was much cleaner. 

\paragraph{Computation of Posterior Uncertainty Interval}
Each of the plots in \textbf{Figure \ref{fig: experiments}} shows the $2\sigma$ uncertainty interval on $\Xs$ of a Gaussian process Bayesian adversary with prior covariance $\Sigma$ and any mean function

The posterior covariance is computed using standard formulas for linear Gaussian systems. Knowing that $Z = X + G$, we may write the joint precision matrix $\Lambda$ (inverse of covariance matrix) of $(X,Z)$ as 
\begin{align*}
	\Lambda^{(X,Z)}
	&= 
	\begin{bmatrix}
		\Sigma^{-1} + {\Sigmag}^{-1} & -{\Sigmag}^{-1} \\
		-{\Sigmag}^{-1} & {\Sigmag}^{-1} 
	\end{bmatrix}
\end{align*}

It is then a well known result that the conditional covariance matrix is given by 
\begin{align*}
	\Sigma_{x|z} &= \Lambda_{xx}^{-1}  \\
	&= \big(\Sigma^{-1} + {\Sigmag}^{-1}\big)^{-1}
\end{align*}
This provides the posterior covariance of all locations $X$ given any released trace $Z$ that uses a Gaussian mechanism with covariance $\Sigmag$. Note that the CIP guarantee naturally keeps posterior uncertainty large since the posterior density at any two $x_s$ close together must be similar. For these Gaussian posteriors, $2 \sigma$ tells us the adversary's 68\% confidence interval on $\Xs$ after obvserving $Z$. 

For basic secrets (one location), we simply report twice the posterior standard deviation at the sensitive index $i$, given by 
\begin{align*}
2 \sqrt{ \Sigma_{{x|z, ii}} } \ .
\end{align*}  
For compound secrets involving multiple locations the posterior distribution is a length $|\Is|$ multivariate normal with covariance $\Sigma_{x|z, ss}$. Intuitively, we wish to find the direction of the vector $\Xs$ in which the posterior interval is the \emph{shortest}. This is the worst case posterior interval on the compound secret. We do this by reporting 
\begin{align*}
	2 \sqrt{\text{mineig}\  \Sigma_{{x|z, ss}}} \ .
\end{align*}

\subsection{Discussion of GP Conditional Prior Class}
\label{apx: GP prior class}

Recall that a conditional prior class requires for any $P_{\calP_i}, P_{\calP_j} \in \Theta$ that 
\begin{align*}
	P_{\calP_i}(\Xu | \Xs = x_s)
	&= P_{\calP_j}(\Xu + c_{ij\Is}^u | \Xs = x_s + c_{ij\Is}^s)
\end{align*}
for all $x_s$. Notice that the mapping $(x_s, x_s') + c_{ij\Is}^s$ is a bijection from $\Spairs$ onto itself. As such, each pair of conditional distributions, 
\begin{align*}
	\Big(P_{\calP_j}(\Xu | \Xs = x_s), P_{\calP_j}(\Xu | \Xs = x_s')\Big)
\end{align*}
induced by $(x_s, x_s') \in \Spairs$ is a mean-shifted version of the pair of distributions 
\begin{align*}
	\Big(P_{\calP_i}(\Xu | \Xs = x_s - c_{ij\Is}^s), P_{\calP_i}(\Xu | \Xs = x_s' - c_{ij\Is}^s)\Big)
\end{align*}
induced by $(x_s , x_s' ) - c_{ij\Is}^s \in \Spairs$. Since the R\'enyi divergence between two distributions and two mean-shifted versions thereof is unchanged, we may use one additive noise mechanism for all priors in class $\Theta$.  

To see how this applies to the GP prior class, recall the formula for a conditional multivariate Gaussian distribution: 
\begin{align*}
	P(\Xu | \Xs = x_s)
	&= \calN(\mu_{u|s} , \Sigma_{u|s})
\end{align*}
where, 
\begin{align*}
		\mu_{u|s} &= \mu_u + \Sigma_{us} \Sigma_{ss}^{-1} (x_s - \mu_s) \\
		\Sigma_{u|s} &= \Sigma_{uu} - \Sigma_{us}\Sigma_{ss}^{-1} \Sigma_{su}
\end{align*}
A GP prior class includes all GP distributions with a fixed kernel $k(t_i, t_j)$ and any mean function $\mu(t)$. For a fixed set of time points, this corresponds to a fixed covariance matrix $\Sigma$ and any mean parameters $\bmu$: 
\begin{align*}
	X \sim \calN(\bmu, \Sigma)
\end{align*}

Let $P_{\calP_i} = \calN(\bar{\bmu}, \Sigma)$ and $P_{\calP_j} = \calN(\hat{\bmu}, \Sigma)$, then conditioned on some sensitive points $\Xs$ the distribution on $\Xu$ has the same covariance $\Sigma_{u|s}$ and conditional means 
\begin{align*}
	\bar{\mu}_{u|s}
	&= \bar{\mu}_u + \Sigma_{us} \Sigma_{ss}^{-1} (x_s - \bar{\mu}_s) \\
	&= (\bar{\mu}_u - \Sigma_{us} \Sigma_{ss}^{-1}\bar{\mu}_s) + \Sigma_{us} \Sigma_{ss}^{-1} x_s \\
	\hat{\mu}_{u|s}
	&= \hat{\mu}_u + \Sigma_{us} \Sigma_{ss}^{-1} (x_s - \hat{\mu}_s) \\
	&= (\hat{\mu}_u - \Sigma_{us} \Sigma_{ss}^{-1}\hat{\mu}_s) + \Sigma_{us} \Sigma_{ss}^{-1} x_s 
\end{align*}
which implies that the conditional distributions are identical up to a mean shift for the \emph{same} $x_s$ value. 
\begin{align*}
	P_{\calP_i}(\Xu | \Xs = x_s)
	&= P_{\calP_j}(\Xu + c_{ij\Is}^u | \Xs = x_s)
\end{align*}
for all $x_s$. Here, $c_{ij\Is}^u = (\bar{\mu}_u - \Sigma_{us} \Sigma_{ss}^{-1}\bar{\mu}_s) - (\hat{\mu}_u - \Sigma_{us} \Sigma_{ss}^{-1}\hat{\mu}_s)$, and $c_{ij\Is}^s = 0$. 

To see how this allows a single additive mechanism to work for all mean functions, notice that we also have 
\begin{align*}
	P_{\calP_i}(\Xu | \Xs = x_s')
	&= P_{\calP_j}(\Xu + c_{ij\Is}^u | \Xs = x_s')
\end{align*}
for $x_s'$, so the divergences 
\begin{align*}
	D_\lambda \binom{P_{\calP_i}(\Xu | \Xs = x_s)}{P_{\calP_i}(\Xu | \Xs = x_s')}
	&= D_\lambda \binom{P_{\calP_j}(\Xu + c_{ij\Is}^u | \Xs = x_s)}{P_{\calP_j}(\Xu + c_{ij\Is}^u | \Xs = x_s')} \\
	&= D_\lambda \binom{P_{\calP_j}(\Xu  | \Xs = x_s)}{P_{\calP_j}(\Xu  | \Xs = x_s')}
\end{align*}
are equal. The same goes for the noisy trace $\Xu + \Zu | \Xs = x_s$, when $Z$ is drawn independently of $X$, allowing us to bound privacy loss for all $P \in \Theta$.

%
%
%

\end{document}